\documentclass{article}

\PassOptionsToPackage{numbers, compress}{natbib}

\usepackage[dblblindworkshop,final]{neurips_2026}
\usepackage{amssymb}
\workshoptitle{Geometric Distributional Deep Learning}
\usepackage{amsfonts}
\usepackage{amsmath}
\usepackage{amsthm}
\usepackage{rotating}
\usepackage[acronym]{glossaries}
\usepackage{nicefrac}
\usepackage{float}
\usepackage{xcolor}
\usepackage{algorithm}%
\usepackage{algorithmicx}%
\usepackage{algpseudocode}%
\usepackage{microtype}
\usepackage[english]{babel}

\usepackage{pdflscape}
\usepackage{float}
\usepackage{multirow}
\usepackage{booktabs}
\newtheorem{theorem}{Theorem}[section]

\usepackage{subcaption}
\usepackage{graphicx}

\newacronym{gd}{GD}{Gradient Descent}
\newacronym{ot}{OT}{Optimal Transport}
\newacronym{dr}{DR}{Dimensionality Reduction}
\newacronym{gw}{GW}{Gromov-Wasserstein}
\newacronym{mds}{MDS}{Multidimensional Scaling}
\newacronym{isomap}{Isomap}{Isometric Mapping}
\newacronym{gw-mds}{GW-MDS}{Gromov-Wasserstein MDS}
\newacronym{pca}{PCA}{Principal Component Analysis}
\newacronym{ewca}{EWCA}{Entropic Wasserstein Component Analysis}
\newacronym{srgw}{srGW}{Semi-Relaxed Gromov-Wasserstein}
\newacronym{srgw-mds}{srGW-MDS}{Semi-Relaxed Gromov-Wasserstein MDS}
\newacronym{kl}{KL}{Kurdyka–Łojasiewicz}

\newtheorem{proposition}[theorem]{Proposition}

\newcommand{\argmin}[1]{\underset{#1}{\text{argmin}}\,}
\newcommand{\argmax}[1]{\underset{#1}{\text{argmax}}\,}

\usepackage[utf8]{inputenc} 
\usepackage[T1]{fontenc}    
\usepackage{hyperref}       
\usepackage{url}            
\usepackage{booktabs}       
\usepackage{amsfonts}       
\usepackage{nicefrac}       
\usepackage{microtype}      
\usepackage{xcolor}         

\title{Gromov--Wasserstein Distillation for Inductive Multi-View Embedding}

\author{%
  Rafael Pereira Eufrazio \\
  Instituto Federal do Ceará (IFCE)\\
  Canindé, Ceará, Brazil\\
  Federal University of Ceará (UFC)\\
  Fortaleza, Ceará, Brazil\\
  \texttt{rafael.eufrazio@ifce.edu.br}
  \And
  Eduardo Fernandes Montesuma \\
  Sigma Nova\\
  Paris, France\\
  \texttt{eduardo.montesuma@sigmanova.ai}
  \And
  Charles Casimiro Cavalcante \\
  Federal University of Ceará (UFC)\\
  Fortaleza, Ceará, Brazil\\
  \texttt{charles@ufc.br}
}

\begin{document}

\maketitle

\begin{abstract}
Gromov--Wasserstein multidimensional scaling (GW-MDS) learns
low-dimensional representations from relational data but remains
transductive, providing no explicit mapping for unseen samples. We introduce
an inductive framework based on barycentric distillation. A GW-MDS teacher
learns a latent support and an optimal transport plan from the training data,
and barycentric projection converts the resulting coupling into
sample-aligned targets. A neural student then learns an explicit
out-of-sample mapping, avoiding additional relational-matrix construction
and GW optimization at inference. We formulate the approach for single-view data and extend it to
Mean-GWMDS and Multi-GWMDS teachers through consensus and
selected-projection targets learned by a multi-view student with
view-specific encoders. We also investigate a
direct neural baseline trained solely with a GW objective. Experiments on
synthetic and real-world data using Euclidean, geodesic, and cosine
relations show that the distilled models preserve the teacher geometry on
unseen samples and consistently outperform direct neural GW training in
sample-indexed relational preservation. These results establish barycentric projection as an effective bridge
between transductive GW embeddings and inductive neural mappings.
\end{abstract}

\section{Introduction}
\label{sec:introduction}

Dimensionality reduction seeks compact representations that preserve the
relevant structure of high-dimensional data. This problem becomes more
challenging in multi-view settings, where the same observations are
described through heterogeneous features, dimensions, or measurement
modalities. In this context, comparing coordinates directly may be
inappropriate. The \gls{gw} discrepancy provides a natural alternative
because it compares within-domain relational structures without requiring
the domains to share a common ambient space
\cite{memoli2011gromov,peyre2019computational}.

GW-MDS uses this principle to optimize a low-dimensional support whose
induced geometry matches that of the input data
\cite{Eufrazio,eufrazio2026nonlinear, eufrazio2026structure, eufrazio2026gromovwassersteinmethodsmultiviewrelational} . Despite its flexibility, GW-MDS is
transductive: it optimizes coordinates only for the observed samples and
does not learn a mapping for unseen data. Moreover, the correspondence
between the input samples and the optimized latent support is represented
by a transport plan and need not coincide with the identity assignment.

Consequently, obtaining an inductive extension is not as simple as
training a neural network directly with a GW objective. Such an objective
compares two unindexed relational structures after optimizing their
coupling. A small GW loss may therefore indicate structural agreement
under a non-identity correspondence, without ensuring that the coordinate
predicted for each sample matches its sample-indexed position in the
teacher representation.

We address this ambiguity through a teacher--student framework based on
barycentric distillation. Building upon the standard optimal transport mapping estimation proposed by Seguy et al. \cite{seguy2017large}, the coupling learned by a transductive GW-MDS
teacher is used to project the optimized latent support onto sample-aligned
targets. A neural student then learns to predict these targets from the
original observations. Once trained, it embeds unseen samples through a
forward pass, without constructing new pairwise relational matrices or
solving another GW problem. The same principle is extended to corresponding
multi-view data using Mean-GWMDS and Multi-GWMDS \cite{eufrazio2026structure}.

Our main contributions are:
\begin{itemize}
    \item a barycentric distillation mechanism that resolves the
    sample-correspondence ambiguity of GW embeddings and converts
    transductive representations into explicit inductive mappings;
    \item single-view and multi-view inductive formulations derived from
    GW-MDS, Mean-GWMDS and Multi-GWMDS teachers; and
    \item an evaluation on synthetic and real-world data across multiple
    relational geometries, showing consistent improvements over direct neural GW training and competitive or superior performance relative to PCA-based inductive baselines.
\end{itemize}

The remainder of this paper is organized as follows.
Section~\ref{sec:Background} reviews the GW discrepancy, GW-MDS, its
multi-view extensions, and barycentric projections.
Section~\ref{sec:methods} introduces the proposed barycentric distillation
framework for single-view and multi-view data, provides its theoretical
justification, and summarizes the complete procedure.
Section~\ref{sec:experiments} describes the experimental protocol and
reports the results on synthetic and real-world datasets.
Finally, Section~\ref{sec:conclusion} presents the main conclusions and
directions for future work. Additional proofs, results, and visualizations
are provided in the supplementary material.

\section{Background}
\label{sec:Background}

This section briefly reviews the Gromov--Wasserstein discrepancy and its use in relational dimensionality reduction.

\subsection{Gromov--Wasserstein Discrepancy}
\label{sec:gw_background}

Consider two discrete relational spaces
\(
\mu=\sum_{i=1}^{n}a_i\delta_{x_i}
\)
and
\(
\nu=\sum_{j=1}^{m}b_j\delta_{y_j},
\)
with probability vectors \(a\in\Delta_n\) and \(b\in\Delta_m\), and
pairwise dissimilarity matrices \(D_X\in\mathbb{R}^{n\times n}\) and
\(D_Y\in\mathbb{R}^{m\times m}\). Their admissible transport plans form
the polytope
\(
    \Pi(a,b)
    =
    \left\{
        T\in\mathbb{R}_{+}^{n\times m}:
        T\mathbf{1}_m=a,\;
        T^\top\mathbf{1}_n=b
    \right\}.
    \label{eq:transport_polytope}
\)
The squared Gromov--Wasserstein (GW) discrepancy is
\begin{equation}
\begin{split}
    \operatorname{GW}^{2}
    \left((D_X,a),(D_Y,b)\right)
    =
    \min_{T\in\Pi(a,b)}
    \sum_{i,i'=1}^{n}
    \sum_{j,j'=1}^{m}
    \left(
        (D_X)_{ii'}-(D_Y)_{jj'}
    \right)^2
    T_{ij}T_{i'j'}.
\end{split}
\label{eq:gw_background}
\end{equation}
Unlike standard optimal transport, GW compares internal relational
structures and therefore does not require the supports to share the
same ambient space. The optimal plan \(T^\star\) encodes the resulting
structural correspondence
\cite{memoli2011gromov,peyre2019computational, peyre2016gromov}.

\subsection{Gromov--Wasserstein Multidimensional Scaling}
\label{sec:gwmds_background}

Classical multidimensional scaling learns low-dimensional coordinates
whose pairwise distances approximate an input dissimilarity matrix under
a fixed pointwise correspondence. Recent optimal-transport formulations
have related dimensionality reduction to Gromov--Wasserstein and
semi-relaxed Gromov--Wasserstein problems
\cite{clark2024generalized,van2024distributional}. GW-MDS replaces this fixed
correspondence with an optimized transport plan
\cite{Eufrazio,eufrazio2026nonlinear}.

Given a dataset
\(X=[x_1,\ldots,x_n]^\top\in\mathbb{R}^{n\times p}\),
a relational matrix \(D_X\in\mathbb{R}^{n\times n}\), and an embedding
dimension \(d\), GW-MDS learns a latent support
\(Z=[z_1,\ldots,z_n]^\top\in\mathbb{R}^{n\times d}\) by solving
\begin{equation}
    Z^\star
    \in
    \argmin{Z\in\mathbb{R}^{n\times d}}
    \operatorname{GW}^{2}
    \left(
        (D_X,a),(D_Z,b)
    \right),
    \qquad
    (D_Z)_{jj'}
    =
    \lVert z_j-z_{j'}\rVert_2,
    \label{eq:gwmds_background}
\end{equation}
where \(a=b=\frac{1}{n}\mathbf{1}_n\) in the full-support setting
considered here. The matrix \(D_X\) may encode Euclidean, cosine,
geodesic, or other application-dependent dissimilarities. The
transport plan and latent coordinates are typically estimated by
alternating between the GW transport problem and gradient-based updates
of \(Z\). Although \(Z^\star\) preserves the relational structure of the input,
its rows are not intrinsically aligned with the input indices, since
their correspondence is mediated by the optimal plan
\(T^\star\in\Pi(a,b)\). A sample-indexed representation is obtained
through the barycentric projection
\begin{equation}
    \widetilde Y^\star
    =
    \mathcal{B}_{T^\star}(Z^\star)
    =
    \operatorname{Diag}(a)^{-1}T^\star Z^\star,
    \qquad
    \widetilde y_i^\star
    =
    \frac{1}{a_i}
    \sum_{j=1}^{n}T_{ij}^\star z_j^\star.
    \label{eq:barycentric_projection}
\end{equation}
Thus, \(\widetilde y_i^\star\) is the transport-weighted barycenter of
the latent points associated with the source observation \(x_i\), making
\(\widetilde Y^\star\) explicitly aligned with the original samples.

Nevertheless, GW-MDS remains transductive: it optimizes coordinates for
the observed dataset but does not learn an explicit mapping for unseen
observations. Section~\ref{sec:methods} addresses this limitation by
distilling the barycentric representation into an inductive neural
mapping.

\subsection{Multi-View Relational Data}
\label{sec:multiview_background}

In the corresponding multi-view setting, the same \(n\) observations
are represented through \(V\) views,
\(
    \mathcal{X}
    =
    \left\{
        X^{(1)},\ldots,X^{(V)}
    \right\}, \)
    \(
    X^{(v)}\in\mathbb{R}^{n\times p_v}.
    \label{eq:multiview_data}
\)
The rows of all views refer to the same samples in the same order,
although their feature dimensions \(p_v\) and ambient spaces may differ.
Each view induces a relational matrix
\(D_X^{(v)}\in\mathbb{R}^{n\times n}\). Consequently, different
views may encode complementary relational structures for the same
observations~\cite{yu2025review,qin2025survey, chowdhury2025deep, xu2013survey}.

Multi-view relational embedding seeks a shared low-dimensional
representation that exploits these complementary structures without
requiring coordinate-wise comparability across views. GW-based methods
are naturally suited to this setting because they compare intra-view
relations rather than the original feature coordinates. However, their
resulting embeddings remain transductive, motivating the inductive
multi-view formulations introduced in Section~\ref{sec:methods}.

\section{Proposed Methods}
\label{sec:methods}

We transform transductive GW-MDS into an inductive mapping through
teacher--student distillation. The teacher transport plan is used to
construct sample-aligned barycentric targets, which supervise a neural
student. We first present the single-view formulation and then extend it to corresponding multi-view data.

\subsection{Single-View Inductive GW-MDS}
\label{sec:single_inductive}

Let
\(X=[x_1,\ldots,x_n]^\top\in\mathbb{R}^{n\times p}\)
be the training data and \(D_X\in\mathbb{R}^{n\times n}\) its relational
dissimilarity matrix. We divide \(D_X\) by its maximum entry and reuse the
notation \(D_X\) for the normalized matrix. As in
Section~\ref{sec:gwmds_background}, we consider uniform full-support measures, i.e. 
\(a=b=\frac{1}{n}\mathbf{1}_n\).

The GW-MDS teacher then solves Eq.~\eqref{eq:gwmds_background}.
After the final latent-support update, the transport plan is recomputed,
yielding \(T^\star\in\Pi(a,b)\) associated with \(Z^\star\).
Because the rows of \(Z^\star\) are not intrinsically aligned with the
input indices, directly pairing \(x_i\) with \(z_i^\star\) would impose an
arbitrary correspondence. We therefore apply the barycentric projection
introduced in Eq.~\eqref{eq:barycentric_projection} to define the
sample-indexed teacher targets:
\(
    \widetilde Y
    =
    \mathcal{B}_{T^\star}(Z^\star)
    =
    \operatorname{Diag}(a)^{-1}T^\star Z^\star.
    \label{eq:barycentric_targets}
\)
Each row \(\widetilde y_i\) is consequently aligned with the source sample
\(x_i\) and can be used to supervise the neural student.

Although Eq.~\eqref{eq:barycentric_projection} provides sample-indexed
targets, it remains to justify why this representation is appropriate for
distillation. For a fixed transport plan and a fixed coordinate realization
of the latent support, the following proposition shows that the barycentric
projection is the unique minimizer of the quadratic reconstruction cost
induced by the transport plan. It therefore provides a principled
sample-indexed representative of the selected teacher solution.

\begin{proposition}[Optimality of barycentric targets]
\label{eq:barycentric_optimality}
Let \(T\in\Pi(a,b)\) be a fixed transport plan, with \(a_i>0\)
for every \(i\), and let \(Z=[z_1,\ldots,z_m]^\top\). Then the
barycentric projection
\(
\widetilde Y
=
\operatorname{Diag}(a)^{-1}TZ,\) 
\(
\widetilde y_i
=
\frac{1}{a_i}\sum_{j=1}^{m}T_{ij}z_j,
\)
is the unique minimizer of
\(
\mathcal{Q}_T(Y;Z)
=
\sum_{i=1}^{n}\sum_{j=1}^{m}
T_{ij}\lVert y_i-z_j\rVert_2^2.
\)
Moreover, for every \(Y=[y_1,\ldots,y_n]^\top\),
\(
\mathcal{Q}_T(Y;Z)
=
\mathcal{Q}_T(\widetilde Y;Z)
+
\sum_{i=1}^{n}a_i
\lVert y_i-\widetilde y_i\rVert_2^2.
\)
\end{proposition}

The proof of Proposition~\ref{eq:barycentric_optimality} is provided in Appendix~\ref{app:barycentric_optimality}. This uniqueness is conditional on the selected transport plan and on the
coordinate realization of the latent support. If the latent support is
translated, rotated, or reflected, its barycentric projection undergoes the
same global transformation. Therefore, across isometrically equivalent
teacher solutions, the projected representation is determined only up to a
global Euclidean isometry, while its pairwise distances remain unchanged.

\subsubsection{Neural student}

Following the teacher--student distillation paradigm
\cite{hinton2015distilling} and its representation-matching variants
\cite{romero2015fitnetshintsdeepnets}, let
\(f_\theta:\mathbb{R}^{p}\rightarrow\mathbb{R}^{d}\)
denote the neural student. The barycentric targets are standardized using
statistics computed exclusively from the training set. With a slight abuse
of notation, \(\widetilde y_i\) continues to denote the standardized
barycentric target associated with \(x_i\). The student is trained by
minimizing
\(
\mathcal{L}_{\mathrm{distill}}(\theta)
=
\frac{1}{n}
\sum_{i=1}^{n}
\left\|
f_\theta(x_i)-\widetilde y_i
\right\|_2^2.
\label{eq:single_distillation}
\)
As established in
Proposition~\ref{prop:distillation_relational_control} in the Appendix,
under the uniform empirical weights considered in this work, minimizing
this pointwise objective also controls an upper bound on the average
squared discrepancy between the pairwise Euclidean distances induced by
the student predictions and those induced by the standardized barycentric
targets. Thus, although the student is trained through pointwise
regression, its objective is directly connected to the relational
geometry transferred from the teacher. Early stopping is based on a validation subset. Once trained, the student
provides an inductive out-of-sample mapping and embeds an unseen sample
directly as
\(
y_{\mathrm{new}}=f_\theta(x_{\mathrm{new}}).
\)

\subsection{Multi-View Barycentric Distillation}
\label{sec:multiview_inductive}

For the multi-view data introduced in
Section~\ref{sec:multiview_background}, let \(D_X^{(v)}\) denote the
normalized relational matrix of view \(v\), and let
\(\lambda_v\geq 0\), with \(\sum_{v=1}^{V}\lambda_v=1\), denote its
contribution to the shared objective. We use Mean-GWMDS and
Multi-GWMDS \cite{eufrazio2026structure} as complementary multi-view teachers.

\subsubsection{Inductive Mean-GWMDS}

Mean-GWMDS first averages the relational matrices,
\(
\overline D_X=\sum_{v=1}^{V}\lambda_vD_X^{(v)}
\),
and then computes the teacher support and its consensus barycentric target:
\(
Z_{\mathrm{mean}}^\star
\in
\argmin{Z\in\mathbb{R}^{n\times d}}
\operatorname{GW}^{2}
\left(
    (\overline D_X,a),(D_Z,b)
\right), \)
\(\widetilde Y_{\mathrm{mean}}
=
\operatorname{Diag}(a)^{-1}
T_{\mathrm{mean}}^\star
Z_{\mathrm{mean}}^\star
\)
where \(T_{\mathrm{mean}}^\star\) is the corresponding optimal plan.

\subsubsection{Inductive Multi-GWMDS}

Multi-GWMDS instead learns a shared support by jointly minimizing the
view-dependent GW discrepancies:
\(
    Z_{\mathrm{multi}}^\star
    \in
    \argmin{Z\in\mathbb{R}^{n\times d}}
    \sum_{v=1}^{V}
    \lambda_v
    \operatorname{GW}^{2}
    \left(
        (D_X^{(v)},a),(D_Z,b)
    \right).
    \label{eq:multi_teacher}
\)
Each optimal plan \(T^{(v)\star}\) induces a view-specific
sample-indexed projection
\(
    \widetilde Y^{(v)}
    =
    \operatorname{Diag}(a)^{-1}
    T^{(v)\star}Z_{\mathrm{multi}}^\star, \)
    \( v=1,\ldots,V.
    \label{eq:multiview_projections}
\)
These projections are retained separately because they encode distinct
view-dependent correspondences. When a single representation is required, the candidate projections are
evaluated using the Pearson correlation between their induced distance
matrices and the relational structures of the training views. The resulting
view-wise correlations are combined using a prespecified aggregation
criterion. We then select the
projection with the highest aggregated agreement with the training views:
\begin{equation}
    s_v
    =
    \operatorname{Agg}_{u=1,\ldots,V}
    \rho_{\mathrm P}
    \left(
        D_{\widetilde Y^{(v)}},
        D_X^{(u)}
    \right),
    \qquad
    v^\star
    =
    \argmax{v\in\{1,\ldots,V\}}s_v,
\end{equation}
where \(\rho\) is the Pearson correlation between the upper-triangular
entries of the two distance matrices. This selection uses training data
only.

\subsubsection{Multi-view student}

The multi-view student processes each view with a view-specific encoder, concatenates the resulting features, and maps them to the teacher coordinates through a shared output head. Given the Mean-GWMDS target
\(Y^\dagger=\widetilde Y_{\mathrm{mean}}\) or the selected Multi-GWMDS target
\(Y^\dagger=\widetilde Y^{(v^\star)}\), its parameters are learned through
\(
    \min_\theta
    \frac{1}{n}
    \sum_{i=1}^{n}
    \left\|
        f_\theta(\mathbf{x}_i)-y_i^\dagger
    \right\|_2^2,
    \) \(
    \mathbf{x}_i=(x_i^{(1)},\ldots,x_i^{(V)}).
\)
Thus, fusion is performed on learned view-specific features while preserving the sample-aligned supervision provided by the teacher.

\subsection{Direct Neural GW Baseline}
\label{sec:direct_gw}

To assess the contribution of barycentric supervision, we compare distillation
with direct neural minimization of
\(
    \min_\theta
    \sum_{v=1}^{V}
    \lambda_v
    \operatorname{GW}^{2}
    \left(
        (D_X^{(v)},a),
        (D_{f_\theta(\mathcal{X})},b)
    \right),
    \label{eq:direct_multiview_gw}
\) an approach analogous to the optimization of generative networks via Gromov-Wasserstein losses \cite{bunne2019learning}.
The single-view case follows by setting \(V=1\).
At each iteration, the transport plans are recomputed from the current output and treated as fixed during backpropagation. The embeddings are centered and scaled by their diameter using training-set statistics.

\section{Experiments and Discussion}
\label{sec:experiments}

We first evaluate Single-View Inductive GW-MDS on a controlled synthetic manifold and then investigate its multi-view extensions on real-world datasets. 
The source code, notebooks, and data used in the experiments are publicly available on GitHub.\footnote{\url{https://github.com/casarpe/inductive-gw-embedding}}

\subsection{Single-View Evaluation on a Synthetic Manifold}
\label{sec:single_synthetic}

We use an S-curve with \(1{,}250\) samples and noise level \(0.05\),
split into \(1{,}000\) training and \(250\) test samples.
Euclidean, cosine, and geodesic relations are considered, with the
latter computed using \(12\) nearest neighbors. Each relational matrix is
normalized by its maximum. For geodesic evaluation, the teacher uses only the training graph, while test relations are extracted from the test--test block of the full graph, used exclusively for evaluation. The GW-MDS
teacher is optimized for \(100\) iterations using Adam with learning rate
\(0.1\). The student has two \(64\)-unit ReLU layers and learns the
teacher's barycentric targets. Direct neural GW uses the same architecture,
while PCA is fitted only to the training data. Evaluation uses Pearson
correlation, Spearman correlation, trustworthiness (\(k=10\)), and
scale-adjusted stress.

\begin{table*}[h]
\centering
\caption{Out-of-sample relational preservation on the S-curve. Bold indicates the best result for each
geometry.}
\label{tab:single_view_scurve}
\small
\begin{tabular}{llcccc}
\toprule
Geometry & Method
& Pearson \(r\) \(\uparrow\)
& Spearman \(\rho\) \(\uparrow\)
& Trust. \(\uparrow\)
& Stress \(\downarrow\) \\
\midrule
Euclidean
& Ind. GW-MDS
& 0.8869 & 0.8911 & \textbf{0.9300} & 0.2003 \\
& Direct neural GW
& 0.7533 & 0.7462 & 0.8387 & 0.2791 \\
& PCA
& \textbf{0.8959} & \textbf{0.9075} & 0.9149
& \textbf{0.1971} \\
\midrule
Geodesic
& Ind. GW-MDS
& \textbf{0.9980} & \textbf{0.9974}
& \textbf{0.9981} & \textbf{0.0337} \\
& Direct neural GW
& 0.7142 & 0.7248 & 0.9233 & 0.3905 \\
& PCA
& 0.7314 & 0.7601 & 0.9369 & 0.3575 \\
\midrule
Cosine
& Ind. GW-MDS
& \textbf{0.9513} & \textbf{0.9635}
& \textbf{0.9963} & \textbf{0.2644} \\
& Direct neural GW
& 0.1904 & 0.2738 & 0.7772 & 0.6648 \\
& PCA
& 0.6817 & 0.7310 & 0.8786 & 0.4613 \\
\bottomrule
\end{tabular}
\end{table*}

Table~\ref{tab:single_view_scurve} shows that geodesic relations
provide the best overall performance, with Inductive GW-MDS reaching
Pearson and Spearman correlations of \(0.9980\) and \(0.9974\),
trustworthiness of \(0.9981\), and stress of \(0.0337\). Cosine relations
also yield strong correlation and neighborhood preservation, whereas,
under Euclidean relations, Inductive GW-MDS performs comparably to PCA
and achieves the highest trustworthiness.

\begin{figure}[!h]
\centering
\begin{subfigure}[t]{0.32\linewidth}
    \centering
\includegraphics[width=\linewidth]{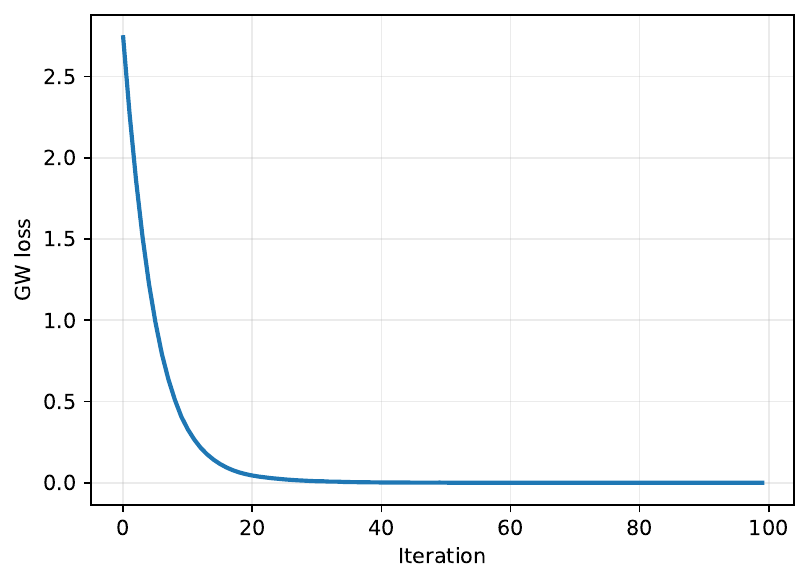}
    \caption{GW-MDS teacher loss.}
    \label{fig:single_teacher_loss}
\end{subfigure}
\hfill
\begin{subfigure}[t]{0.32\linewidth}
    \centering
    \includegraphics[width=\linewidth]{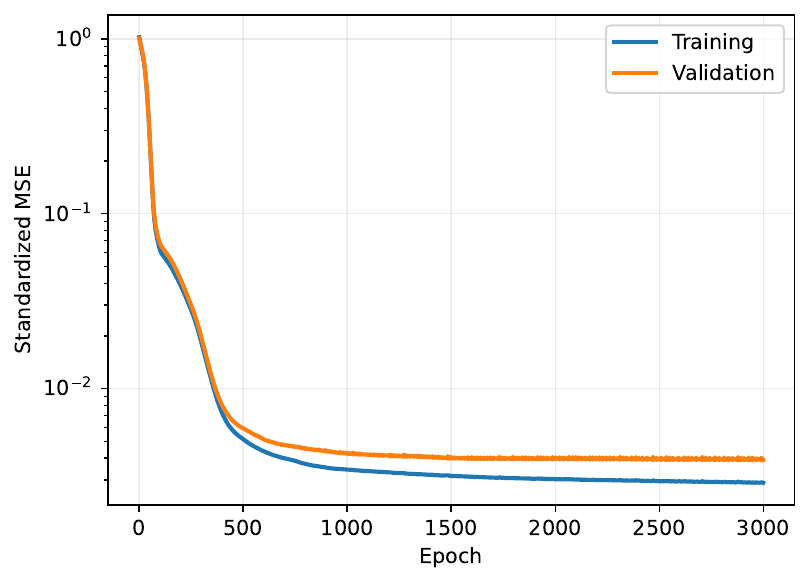}
    \caption{Student distillation loss.}
    \label{fig:single_student_loss}
\end{subfigure}
\hfill
\begin{subfigure}[t]{0.32\linewidth}
    \centering
    \includegraphics[width=\linewidth]{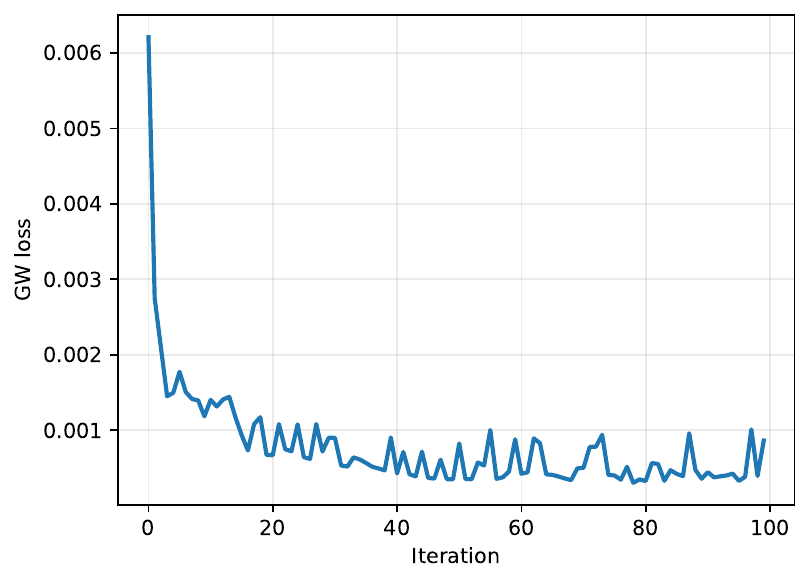} 
    \caption{Direct neural GW loss.}
    \label{fig:single_direct_loss}
\end{subfigure}
\medskip
\begin{subfigure}[t]{0.32\linewidth}
    \centering
    \includegraphics[width=\linewidth]{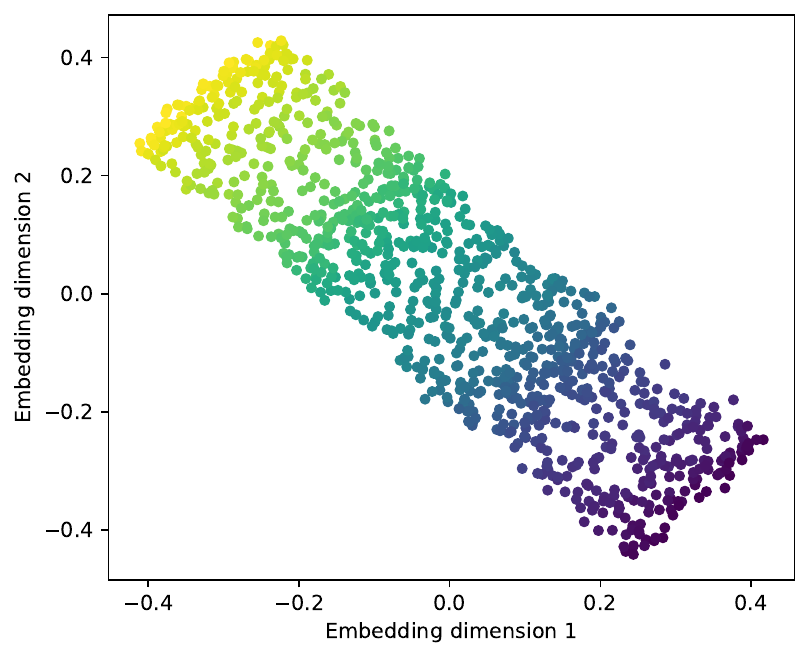}  
    \caption{Teacher training embedding.}
    \label{fig:single_teacher_embedding}
\end{subfigure}
\hfill
\begin{subfigure}[t]{0.32\linewidth}
    \centering
    \includegraphics[width=\linewidth]{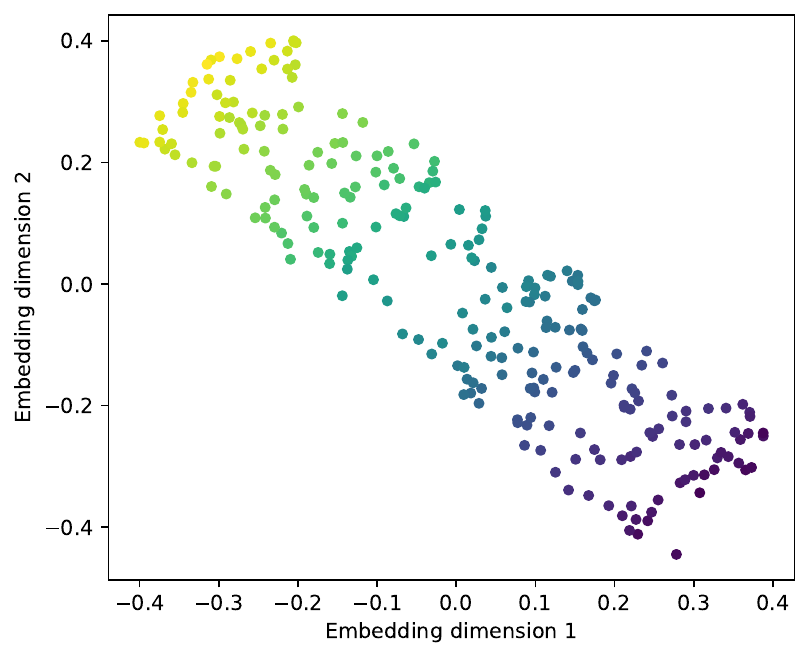}
    \caption{Distilled test embedding.}
    \label{fig:single_student_embedding}
\end{subfigure}
\hfill
\begin{subfigure}[t]{0.32\linewidth}
    \centering
    \includegraphics[width=\linewidth]{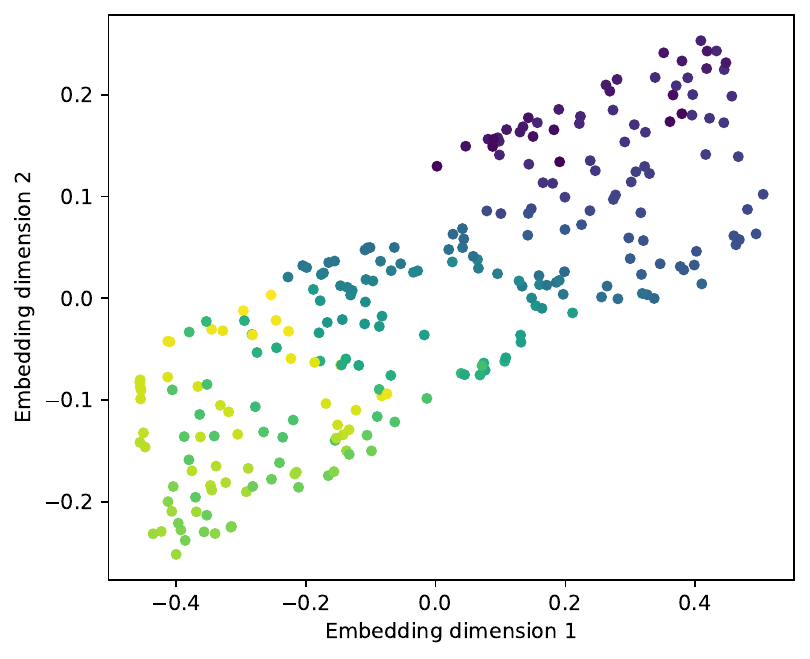}
    \caption{Direct GW test embedding.}
    \label{fig:single_direct_embedding}
\end{subfigure}
\caption{Single-view results on the S-curve using geodesic
dissimilarities.}
\label{fig:single_scurve}
\end{figure}

Inductive GW-MDS consistently outperforms direct neural GW. For the
geodesic setting, Pearson correlation increases from \(0.7142\)
to \(0.9980\), while stress decreases from \(0.3905\) to \(0.0337\).
This indicates that a low direct GW loss does not necessarily recover the
sample-indexed organization, since the optimized coupling may encode a
non-identity correspondence. Barycentric distillation removes this
ambiguity by providing sample-aligned targets that can be extended to
unseen data.

Figure~\ref{fig:single_scurve} shows that both neural approaches learn
structured representations of the unseen samples. Differences in global
orientation should not be interpreted as errors, since GW-based objectives
are invariant to rotations and reflections. The relevant distinction lies
in the local sample organization, the distilled student more closely
reproduces the structure induced by the teacher's barycentric projection,
whereas the direct baseline obtains a low GW loss without recovering the
same sample-indexed arrangement. This qualitative result suggests that a
low direct GW objective does not necessarily ensure preservation of the
teacher's sample-level correspondence.

\subsection{Real-World Datasets}
We evaluate the proposed framework on two real-world datasets from distinct scientific domains: ERA5 and rMD17--Aspirin.

\paragraph{ERA5.}
ERA5 is the fifth-generation atmospheric reanalysis produced by the
European Centre for Medium-Range Weather Forecasts
(ECMWF)~\cite{ hersbach2023era5singlelevels}.
We use hourly data on single levels obtained from the Copernicus Climate
Change Service Climate Data Store (CDS). We consider hourly observations from
January 2024 over a regular grid containing \(23\times21=483\) spatial
locations and \(744\) time instants. Each spatial location constitutes
one sample, represented by its complete hourly time series. Four
co-registered views are constructed from 2-m temperature
(\texttt{t2m}), 2-m dew-point temperature (\texttt{d2m}), surface
pressure (\texttt{sp}), and total precipitation (\texttt{tp}).
Consequently, each view is represented by a matrix
\(X^{(v)}\in\mathbb{R}^{483\times744}\). Although the views share the
same locations and sample ordering, each atmospheric variable induces
a distinct relational geometry.

\paragraph{rMD17--Aspirin.}
We select \(1{,}000\) approximately equidistant conformations from the temporally ordered rMD17 aspirin trajectory, which contains \(100{,}000\) conformations of a \(21\)-atom molecule~\cite{christensen2020rmd17}.
Each conformation is represented by two rotation-invariant views: the \(210\) pairwise interatomic distances and the \(231\) upper-triangular
entries of the force Gram matrix \(F_iF_i^\top\), including its diagonal, where \(F_i\in\mathbb{R}^{21\times3}\).

\subsection{Experimental Protocol}
\label{sec:experimental_protocol}

ERA5 and rMD17--Aspirin are split into \(80\%\) training and \(20\%\)
test samples, yielding \(386/97\) locations for ERA5
and \(800/200\) conformations for rMD17--Aspirin. Fifteen percent of each training set is reserved for validation during student distillation.


For both datasets, Euclidean, cosine, and geodesic relational
matrices are constructed, with graph-geodesic distances computed from a
\(12\)-nearest-neighbor graph. Each matrix is normalized by its maximum
entry. 

The Mean-GWMDS and Multi-GWMDS teachers are optimized for \(100\) outer
iterations using Adam with learning rate \(0.1\). Each inner GW problem
is solved for at most \(100\) iterations with tolerance \(10^{-5}\). For
Multi-GWMDS, the barycentric projection with the highest mean agreement
across the training views is used for distillation. The students employ
two-layer view-specific encoders with \(64\) ReLU units per layer, followed
by a \(64\)-unit fusion layer and a two-dimensional output. They are
trained against the barycentric targets using Adam (learning rate
\(10^{-3}\), weight decay \(10^{-5}\)) for at most \(3{,}000\) epochs,
with a \(15\%\) validation split and early stopping patience of \(250\)
epochs. Direct GW uses the same architecture, while concatenated PCA
serves as the conventional baseline.

Results are uniformly averaged over the four ERA5 views or the two rMD17--Aspirin views. Higher correlation and trustworthiness and lower stress indicate better preservation.
Additional view-wise results, projection-selection scores, GW objective
values, computational costs, and qualitative analyses are provided in
Appendix~\ref{app:supplementary_analysis}.

\subsection{ERA5 Results}

\begin{table*}[t]
\centering
\caption{Out-of-sample relational preservation on ERA5 and
rMD17--Aspirin. ERA5 results are uniform averages over four views on
\(97\) held-out locations, whereas rMD17 results are averaged over two
views. Bold indicates
the best result for each dataset and relational geometry.}
\label{tab:multiview_inductive_results}
\scriptsize
\setlength{\tabcolsep}{5.2 pt}
\begin{tabular}{llcccccccc}
\toprule
& &
\multicolumn{4}{c}{ERA5} &
\multicolumn{4}{c}{rMD17--Aspirin} \\
\cmidrule(lr){3-6}
\cmidrule(lr){7-10}
Geometry & Method
& \(r\uparrow\) & \(\rho\uparrow\)
& Trust.\(\uparrow\) & Stress\(\downarrow\)
& \(r\uparrow\) & \(\rho\uparrow\)
& Trust.\(\uparrow\) & Stress\(\downarrow\) \\
\midrule

Euclidean
& Ind. Mean-GWMDS
& \(\mathbf{0.7853}\) & \(\mathbf{0.7958}\)
& \(\mathbf{0.8975}\) & \(\mathbf{0.2798}\)
& \(\mathbf{0.4736}\) & \(\mathbf{0.4635}\)
& \(\mathbf{0.6852}\) & \(\mathbf{0.3899}\) \\

& Ind. Multi-GWMDS
& \(0.7470\) & \(0.7572\) & \(0.8643\) & \(0.3137\)
& \(0.4065\) & \(0.4018\) & \(0.6783\) & \(0.4014\) \\

& Concatenated PCA
& \(0.7790\) & \(0.7892\) & \(0.8924\) & \(0.2916\)
& \(0.4256\) & \(0.4266\) & \(0.6761\) & \(0.4364\) \\

& Direct multi-view GW
& \(0.4576\) & \(0.5372\) & \(0.8590\) & \(0.4670\)
& \(0.3540\) & \(0.3552\) & \(0.6531\) & \(0.4142\) \\

\midrule

Geodesic
& Ind. Mean-GWMDS
& \(\mathbf{0.8507}\) & \(\mathbf{0.8439}\)
& \(\mathbf{0.9373}\) & \(\mathbf{0.2603}\)
& \(\mathbf{0.4787}\) & \(\mathbf{0.4701}\)
& \(\mathbf{0.6877}\) & \(\mathbf{0.3897}\) \\

& Ind. Multi-GWMDS
& \(0.7855\) & \(0.7826\) & \(0.9049\) & \(0.3106\)
& \(0.4226\) & \(0.4135\) & \(0.6849\) & \(0.3984\) \\

& Concatenated PCA
& \(0.7879\) & \(0.7825\) & \(0.8765\) & \(0.3131\)
& \(0.4126\) & \(0.4080\) & \(0.6630\) & \(0.4153\) \\

& Direct multi-view GW
& \(0.3409\) & \(0.4062\) & \(0.7303\) & \(0.5349\)
& \(0.3471\) & \(0.3486\) & \(0.6466\) & \(0.4272\) \\

\midrule

Cosine
& Ind. Mean-GWMDS
& \(\mathbf{0.5394}\) & \(\mathbf{0.6582}\)
& \(\mathbf{0.8736}\) & \(\mathbf{0.5641}\)
& \(\mathbf{0.5242}\) & \(\mathbf{0.5132}\)
& \(\mathbf{0.6812}\) & \(\mathbf{0.3926}\) \\

& Ind. Multi-GWMDS
& \(0.4830\) & \(0.5875\) & \(0.8322\) & \(0.5712\)
& \(0.4437\) & \(0.4379\) & \(0.6810\) & \(0.4188\) \\

& Concatenated PCA
& \(0.4799\) & \(0.5959\) & \(0.7971\) & \(0.5890\)
& \(0.4222\) & \(0.4247\) & \(0.6761\) & \(0.4340\) \\

& Direct multi-view GW
& \(0.4416\) & \(0.5597\) & \(0.8070\) & \(0.5997\)
& \(0.2942\) & \(0.2972\) & \(0.6371\) & \(0.4703\) \\

\bottomrule
\end{tabular}
\end{table*}

Table~\ref{tab:multiview_inductive_results} shows that Inductive Mean-GWMDS
achieves the best value for all four metrics under every relational
geometry. Its advantage over concatenated PCA is modest for Euclidean
relations, with improvements of \(0.0063\) in Pearson correlation and
\(0.0118\) in stress, but becomes more pronounced for cosine and
geodesic dissimilarities. The geodesic configuration yields
the strongest overall performance, with \(r=0.8507\),
\(\rho=0.8439\), trustworthiness \(0.9373\), and stress \(0.2603\).
This indicates that geodesic-based relations better capture the nonlinear
spatial organization shared by the meteorological variables.

The view-wise results show that this advantage is not uniform across
variables. Under Euclidean geometry, Inductive Mean-GWMDS obtains
Pearson correlations of \(0.8995\), \(0.9232\), and \(0.8981\) for
temperature, dewpoint temperature, and surface pressure, respectively,
but only \(0.4203\) for total precipitation. Geodesic relations
increase the precipitation correlation to \(0.7021\), while maintaining
correlations above \(0.87\) for the other three views. Thus, their
superior averaged performance arises partly from a better integration
of the precipitation view, whose relational structure is less
compatible with the remaining meteorological variables. Cosine
dissimilarity presents a different limitation for surface pressure:
its Pearson correlation is \(0.2667\), whereas its Spearman correlation
remains \(0.6706\). This discrepancy suggests that the relational
ordering is partially retained, but its linear distance magnitudes are
poorly reproduced, as also reflected by the higher stress.

Inductive Multi-GWMDS remains competitive with PCA, particularly under geodesic geometry, for which it provides higher trustworthiness
and lower stress. The training-only selection criterion chooses the
temperature-induced projection (view 1) for all three geometries. Under geodesic relations, for example, the selected
student obtains Pearson correlations of \(0.9393\), \(0.7821\),
\(0.7772\), and \(0.6436\) across the four views, whereas Inductive
Mean-GWMDS obtains \(0.9194\), \(0.9045\), \(0.8767\), and \(0.7021\).
Mean-GWMDS therefore sacrifices a small amount of temperature-view
fidelity while substantially improving the representation of the other
three views.

Finally, direct multi-view GW performs markedly worse despite reaching
almost the same final GW objective as the Multi-GWMDS teacher. 
Most notably, the direct model attains a slightly lower geodesic
objective but only \(r=0.3409\), compared with \(0.7855\) for Inductive
Multi-GWMDS and \(0.8507\) for Inductive Mean-GWMDS. This result confirms that minimizing the GW objective can recover structural agreement under
an optimized coupling without recovering the required sample-indexed
organization. Barycentric distillation addresses this ambiguity by
transferring the correspondence encoded by the teacher's transport plan
to the inductive student.

\begin{figure*}[t]
\centering

\begin{subfigure}[t]{0.24\textwidth}
    \centering
    \includegraphics[width=\linewidth]{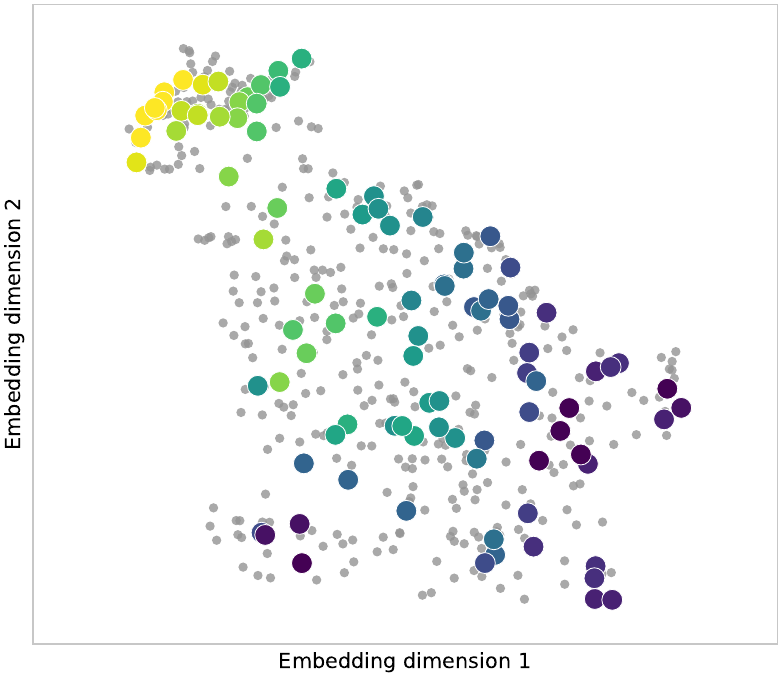}
    \caption{Inductive Mean-GWMDS.}
    \label{fig:era5_joint_mean}
\end{subfigure}
\hfill
\begin{subfigure}[t]{0.24\textwidth}
    \centering
    \includegraphics[width=\linewidth]{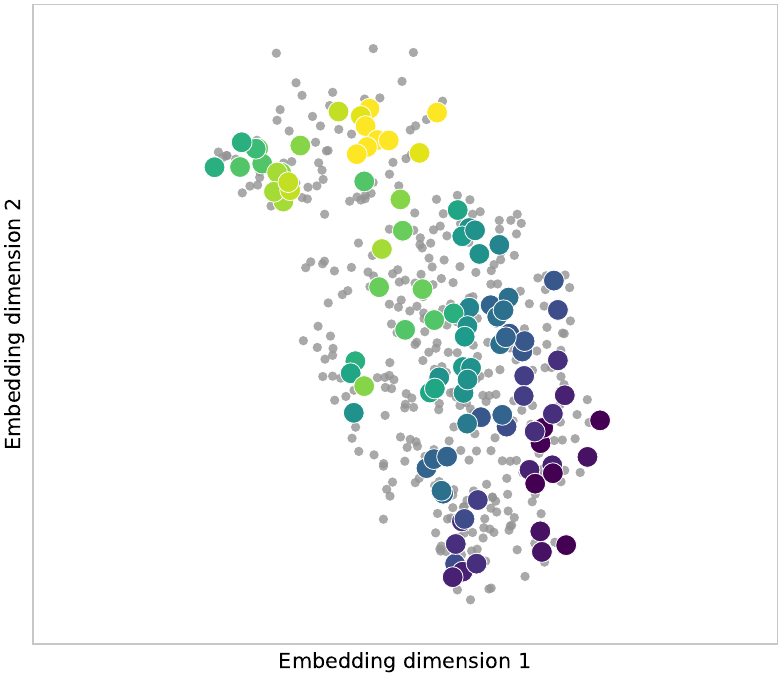}
    \caption{Inductive Multi-GWMDS.}
    \label{fig:era5_joint_multi}
\end{subfigure}
\hfill
\begin{subfigure}[t]{0.24\textwidth}
    \centering
    \includegraphics[width=\linewidth]{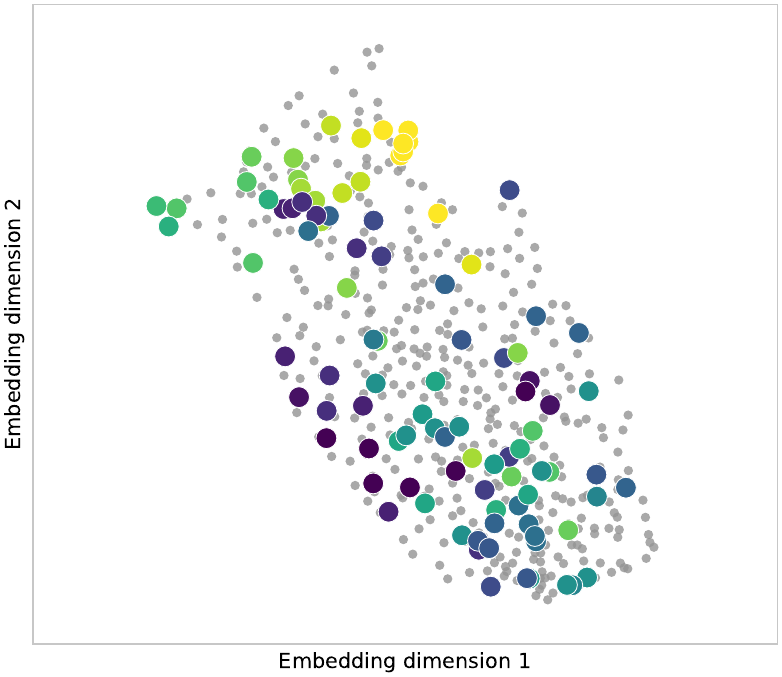}
    \caption{Direct multi-view GW.}
    \label{fig:era5_joint_direct}
\end{subfigure}
\hfill
\begin{subfigure}[t]{0.24\textwidth}
    \centering
    \includegraphics[width=\linewidth]{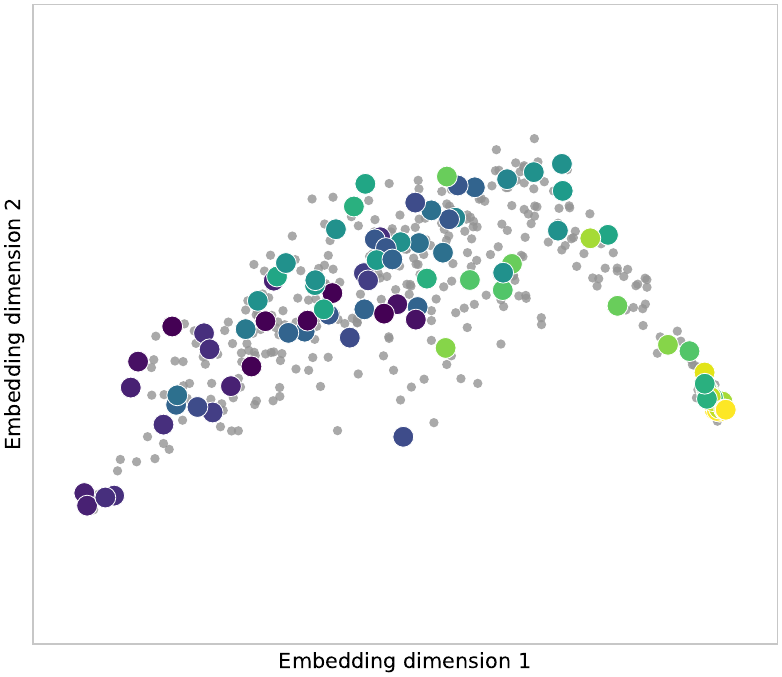}
    \caption{Concatenated PCA.}
    \label{fig:era5_joint_pca}
\end{subfigure}

\caption{Joint training and out-of-sample embeddings for ERA5 under
geodesic relations. Training locations are shown in gray, whereas the
97 held-out locations are colored by latitude using a common color
scale. }
\label{fig:era5_joint_train_test}

\end{figure*}

Figure~\ref{fig:era5_joint_train_test} complements the quantitative
results by showing how the held-out ERA5 locations are positioned
relative to the representations learned from the training data. Both
distilled models place the unseen locations within the structured
regions occupied by the training observations and exhibit a coherent
progression with latitude. Inductive Mean-GWMDS produces the clearest consensus organization.
Inductive Multi-GWMDS also reflects a consensus geometry learned
jointly from all views, although its sample-indexed arrangement is
inherited from the selected barycentric projection. In contrast, the
direct multi-view GW representation is more diffuse and shows greater
mixing among locations with different latitudes.

\subsection{Results on rMD17--Aspirin}
\label{sec:rmd17_results}

The rMD17--Aspirin results in
Table~\ref{tab:multiview_inductive_results} show that Inductive
Mean-GWMDS outperforms all competing methods across every geometry and
metric. Its advantage over concatenated PCA is particularly clear for
cosine dissimilarity, for which the Pearson correlation increases from
\(0.4222\) to \(0.5242\). Cosine dissimilarity provides the highest
global correlations, whereas geodesic relations yield the highest
trustworthiness and lowest stress. The latter differences relative to
Euclidean geometry are small, however, indicating a modest trade-off:
cosine relations favor global distance agreement, while Euclidean and geodesic relations provide slightly better local-neighborhood and
metric preservation.

The averaged results conceal a strong asymmetry between the molecular
views. The training-only selection criterion chooses the projection
induced by the interatomic-distance view under all three geometries.
Under cosine dissimilarity, Inductive Multi-GWMDS consequently attains
correlations of \(0.8204\) and \(0.0669\) with the interatomic-distance
and force-derived structures, respectively, whereas Inductive
Mean-GWMDS obtains \(0.7740\) and \(0.2743\). The same pattern occurs
under Euclidean and geodesic relations: Mean-GWMDS consistently
improves preservation of the force-derived view while retaining high
fidelity to the interatomic-distance view. Its superior averaged
performance therefore results from a more balanced relational
consensus, suggesting that the two molecular structures are only
partially compatible within a shared two-dimensional representation.

The training-set diagnostics further show that distillation closely
preserves the teachers' relational performance. Across the three
geometries, the difference in weighted Pearson correlation between each
teacher and its student is at most \(0.0197\) for Mean-GWMDS and
\(0.0045\) for the selected Multi-GWMDS projection. In contrast, direct
multi-view GW remains inferior despite reaching objective values nearly
identical to those of the Multi-GWMDS teacher; under cosine
dissimilarity, for example, the values are \(0.00920\) and \(0.00919\),
respectively. This near equality rules out insufficient GW optimization
as the main explanation for the performance gap. A low GW loss ensures
structural agreement only under the optimized coupling and does not
guarantee the required sample-indexed organization. 

\begin{figure*}[t]
\centering

\begin{subfigure}[t]{0.24\textwidth}
    \centering
    \includegraphics[width=\linewidth]{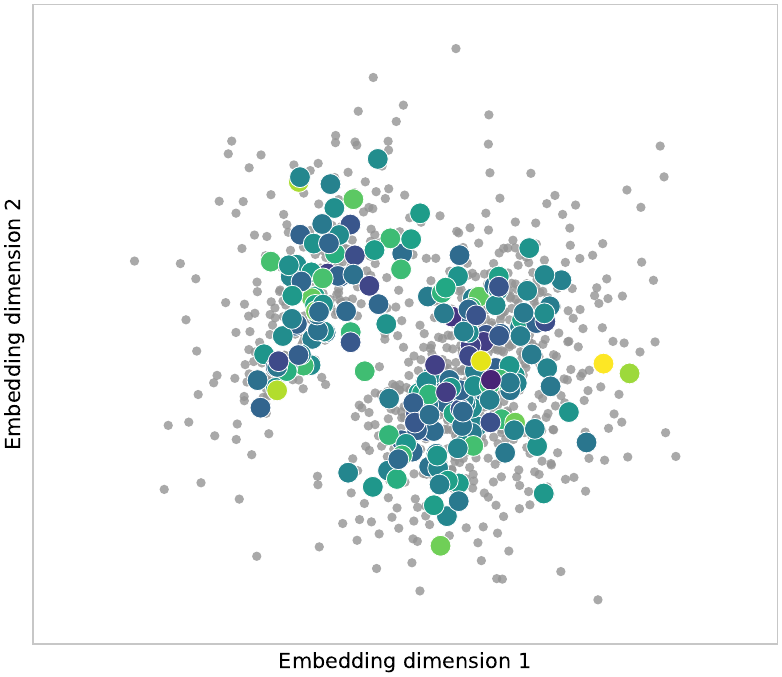}
    \caption{Inductive Mean-GWMDS.}
    \label{fig:aspirin_joint_mean}
\end{subfigure}
\hfill
\begin{subfigure}[t]{0.24\textwidth}
    \centering
    \includegraphics[width=\linewidth]{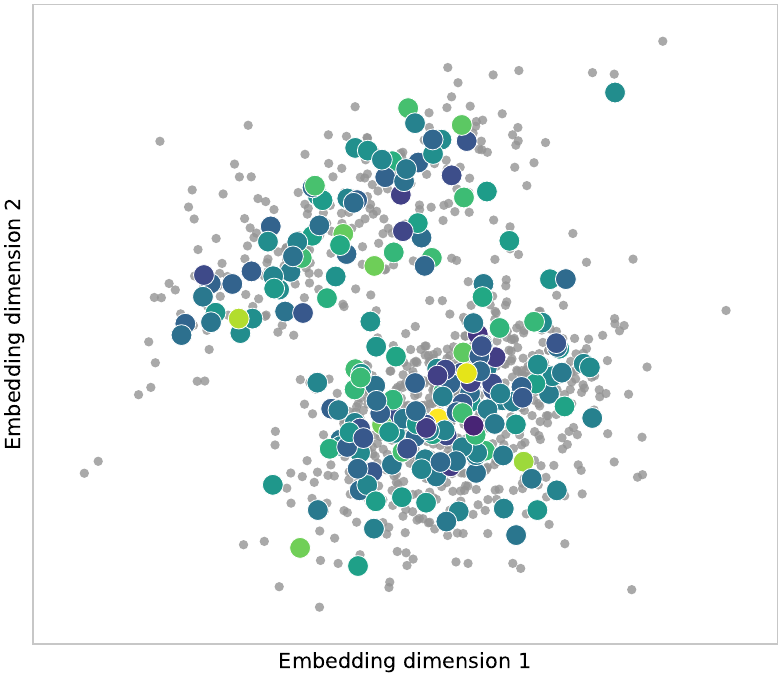}
    \caption{Inductive Multi-GWMDS.}
    \label{fig:aspirin_joint_multi}
\end{subfigure}
\hfill
\begin{subfigure}[t]{0.24\textwidth}
    \centering
    \includegraphics[width=\linewidth]{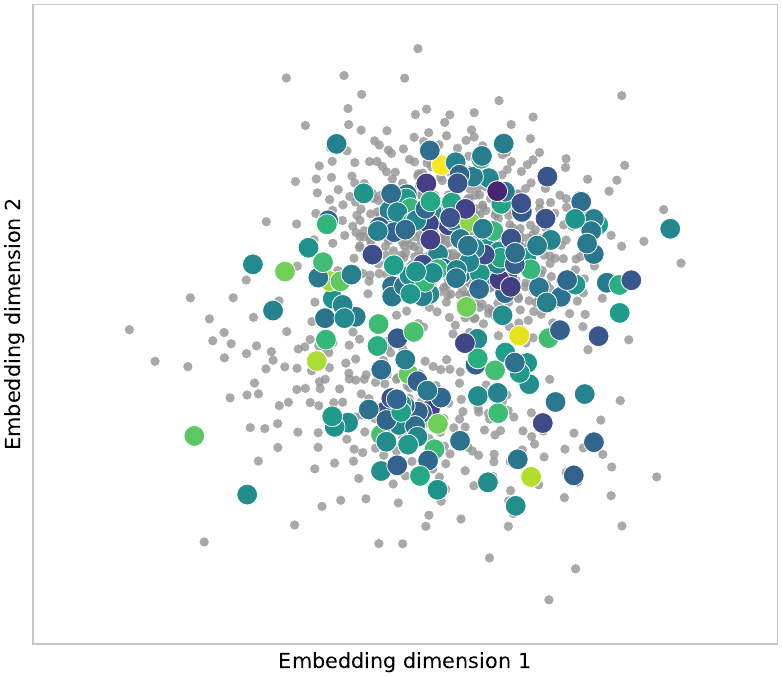}
    \caption{Direct multi-view GW.}
    \label{fig:aspirin_joint_direct}
\end{subfigure}
\hfill
\begin{subfigure}[t]{0.24\textwidth}
    \centering
    \includegraphics[width=\linewidth]{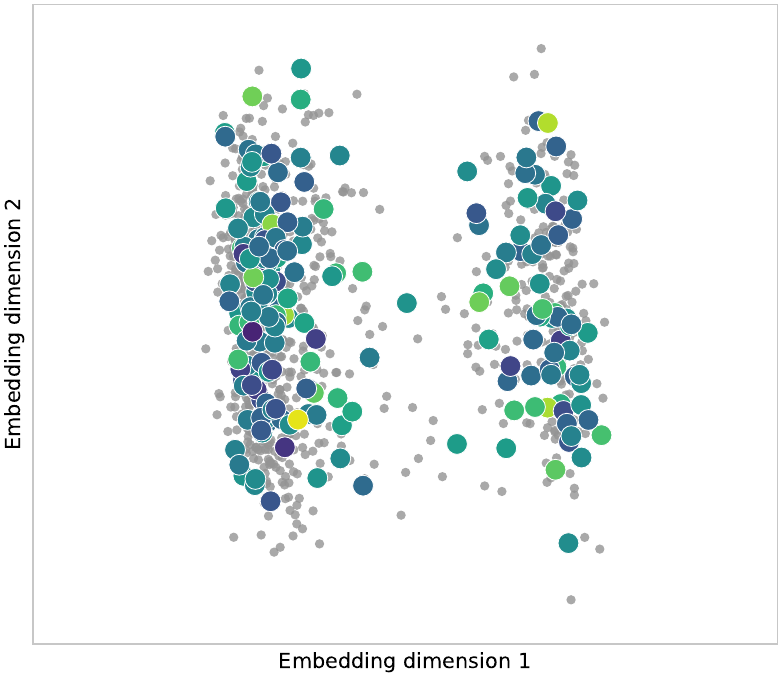}
    \caption{Concatenated PCA.}
    \label{fig:aspirin_joint_pca}
\end{subfigure}

\caption{Joint training and out-of-sample embeddings for
rMD17--Aspirin under cosine relations. Training conformations are
shown in gray, whereas the 200 held-out conformations are colored by
normalized potential energy using a common color scale.}
\label{fig:aspirin_joint_train_test}

\end{figure*}

Figure~\ref{fig:aspirin_joint_train_test} illustrates the joint
training and out-of-sample representations obtained under cosine
relations. The held-out conformations generally occupy regions
supported by the training observations in all four embeddings, but
the resulting organizations differ substantially. Inductive
Mean-GWMDS preserves a branched consensus structure, with the test
conformations distributed coherently along its principal regions.
Inductive Multi-GWMDS instead produces two more clearly separated
groups, reflecting the view-dependent organization inherited from the
selected barycentric projection. The direct multi-view GW embedding
is considerably more diffuse, with much of the data concentrated in a
broad central region, whereas PCA captures a coarse separation into
two dominant branches without explicitly balancing the two relational
views. Potential energy is used only as an
auxiliary coloring variable and not as supervision; therefore, its
partial mixing across the embeddings is expected. As in the ERA5
analysis, the relevant comparison concerns local organization and the
placement of held-out observations relative to the training support,
rather than global orientation or scale.

\section{Conclusion}
\label{sec:conclusion}

We introduced an inductive extension of GW-MDS based on barycentric
distillation. The framework separates relational representation learning
from out-of-sample prediction: a transductive teacher optimizes a latent
support and its GW coupling, the barycentric projection converts this
coupling into sample-aligned targets, and a neural student learns an
explicit mapping to these targets. This principle was developed for
single-view GW-MDS and extended to the multi-view setting using the
consensus target of Mean-GWMDS and the selected-projection target of
Multi-GWMDS.
Experiments on the S-curve, ERA5, and rMD17--Aspirin show that the
distilled students effectively preserve relational structure on unseen
samples and consistently outperform direct neural GW training. In the
multi-view experiments, Inductive Mean-GWMDS achieves the strongest
view-averaged preservation by balancing partially compatible relational
structures, whereas Inductive Multi-GWMDS favors the structure associated
with its training-selected projection. Moreover, direct GW training can
reach objective values close to those of the transductive teacher while
producing substantially worse sample-indexed representations. This
confirms that structural agreement under an optimized coupling does not
by itself ensure the correspondence required for out-of-sample
prediction; the information encoded by the transport plan must also be
transferred. Future work will investigate scalable teacher optimization, a clustering-oriented variant of the proposed framework equipped with mini-batch optimization, and broader evaluations involving larger datasets, repeated data splits, and additional multi-view settings.

\section*{Acknowledgments}

This work was partially supported by CNPq under Grants
308512/2023-5 and 420341/2025-0; CNPq/INCT STREAM
(Signal Processing and Transmission for Environmental Analysis
and Monitoring) under Grant 409179/2024-8; and CAPES,
Brazil, under Finance Code 001.

\bibliographystyle{IEEEtran}

\bibliography{sn-bibliography}

@article{memoli2011gromov,
  title="{Gromov--Wasserstein distances and the metric approach to object matching}",
  author={M{\'e}moli, Facundo},
  journal={Foundations of Computational Mathematics},
  volume={11},
  pages={417--487},
  year={2011},
  publisher={Springer}
}

@article{
van2024distributional,
title="{Distributional Reduction: Unifying Dimensionality Reduction and Clustering with Gromov-Wasserstein}",
author={{H. Van Assel, C{\'e}dric Vincent-Cuaz, Nicolas Courty, R{\'e}mi Flamary, Pascal Frossard, Titouan Vayer}},
journal={Transactions on Machine Learning Research},
issn={2835-8856},
year={2025},
url={https://openreview.net/forum?id=cllm6SS354},
note={}
}

@article{peyre2019computational,
  title="{Computational optimal transport: With applications to data science}",
  author={Peyr{\'e}, Gabriel and Cuturi, Marco},
  journal={Foundations and Trends{\textregistered} in Machine Learning},
  volume={11},
  number={5-6},
  pages={355--607},
  year={2019},
  publisher={Now Publishers, Inc.}
}

@article{flamary2021pot,
  title="{POT: Python Optimal Transport}",
  author={Flamary, R{\'e}mi and Courty, Nicolas and Gramfort, Alexandre and Alaya, Mokhtar Z and Boisbunon, Aur{\'e}lie and Chambon, Stanislas and Chapel, Laetitia and Corenflos, Adrien and Fatras, Kilian and Fournier, Nemo and others},
  journal={Journal of Machine Learning Research},
  volume={22},
  number={78},
  pages={1--8},
  year={2021}
}

@inproceedings{clark2024generalized,
  title="{Generalized dimension reduction using semi-relaxed Gromov-Wasserstein distance}",
  author={Clark, Ranthony A and Needham, Tom and Weighill, Thomas},
  booktitle={Proceedings of the AAAI Conference on Artificial Intelligence},
  volume={39},
  number={15},
  pages={16082--16090},
  year={2025}
}

@inproceedings{Eufrazio,
  title="{A dimensionality reduction technique based on the Gromov-Wasserstein distance}",
  author={ \hspace{0.1pt} Eufrazio, Rafael Pereira and Montesuma, Eduardo Fernandes and Cavalcante, Charles Casimiro},
  booktitle={International Conference on Geometric Science of Information},
  pages={111--120},
  year={2025},
  organization={Springer}
}

@article{yu2025review,
  title={A review on multi-view learning},
  author={Yu, Zhiwen and Dong, Ziyang and Yu, Chenchen and Yang, Kaixiang and Fan, Ziwei and Chen, CL Philip},
  journal={Frontiers of Computer Science},
  volume={19},
  number={7},
  pages={197334},
  year={2025},
  publisher={Springer}
}

@article{chowdhury2025deep,
  title={Deep multi-view clustering: A comprehensive survey of the contemporary techniques},
  author={Chowdhury, Anal Roy and Gupta, Avisek and Das, Swagatam},
  journal={Information Fusion},
  pages={103012},
  year={2025},
  publisher={Elsevier}
}

@article{xu2013survey,
  title={A Survey on Multi-view Learning},
  author={Xu, Chang and Tao, Dacheng and Xu, Chao},
  journal={arXiv preprint arXiv:1304.5634},
  year={2013}
}

@inproceedings{peyre2016gromov,
  title={Gromov-wasserstein averaging of kernel and distance matrices},
  author={Peyr{\'e}, Gabriel and Cuturi, Marco and Solomon, Justin},
  booktitle={International conference on machine learning},
  pages={2664--2672},
  year={2016},
  organization={PMLR}
}

@article{qin2025survey,
  title={A survey on representation learning for multi-view data},
  author={Qin, Yalan and Zhang, Xinpeng and Yu, Shui and Feng, Guorui},
  journal={Neural Networks},
  volume={181},
  pages={106842},
  year={2025},
  publisher={Elsevier}
}

@misc{eufrazio2026structure,
      title={Structure-Preserving Multi-View Embedding Using Gromov-Wasserstein Optimal Transport}, 
      author={Rafael Pereira Eufrazio and Eduardo Fernandes Montesuma and Charles Casimiro Cavalcante},
      year={2026},
      eprint={2604.02610},
      archivePrefix={arXiv},
      primaryClass={stat.ML},
      url={https://arxiv.org/abs/2604.02610}, 
}

@article{eufrazio2026nonlinear,
  title={Nonlinear dimensionality reduction through optimal transport between incomparable spaces},
  author={\hspace{0.01pt} Eufrazio, Rafael Pereira and Montesuma, Eduardo Fernandes and Cavalcante, Charles Casimiro},
  journal={Information Geometry},
  pages={1--36},
  year={2026},
  publisher={Springer}
}

@misc{eufrazio2026gromovwassersteinmethodsmultiviewrelational,
      title={Gromov-Wasserstein Methods for Multi-View Relational Embedding and Clustering}, 
      author={\hspace{0.0001pt}Eufrazio, Rafael Pereira and Eduardo Fernandes Montesuma and Charles Casimiro Cavalcante},
      year={2026},
      eprint={2604.23912},
      archivePrefix={arXiv},
      primaryClass={cs.LG},
      url={https://arxiv.org/abs/2604.23912}, 
}

@article{hinton2015distilling,
  title={Distilling the knowledge in a neural network},
  author={Hinton, Geoffrey and Vinyals, Oriol and Dean, Jeff},
  journal={arXiv preprint arXiv:1503.02531},
  year={2015}
}

@misc{romero2015fitnetshintsdeepnets,
      title={FitNets: Hints for Thin Deep Nets}, 
      author={Adriana Romero and Nicolas Ballas and Samira Ebrahimi Kahou and Antoine Chassang and Carlo Gatta and Yoshua Bengio},
      year={2015},
      eprint={1412.6550},
      archivePrefix={arXiv},
      primaryClass={cs.LG},
      url={https://arxiv.org/abs/1412.6550}, 
}

@misc{hersbach2023era5singlelevels,
  author       = {Hersbach, Hans and Bell, Bill and Berrisford, Paul
                  and Biavati, Gionata and Hor{\'a}nyi, Andr{\'a}s
                  and Mu{\~n}oz Sabater, Joaqu{\'i}n and Nicolas, Julien
                  and Peubey, Carole and Radu, Raluca and Rozum, Irina
                  and Schepers, Dinand and Simmons, Adrian
                  and Soci, Cornel and Dee, Dick
                  and Th{\'e}paut, Jean-No{\"e}l},
  title        = {{ERA5 hourly data on single levels from 1940 to present}},
  year         = {2023},
  howpublished = {Copernicus Climate Change Service (C3S)
                  Climate Data Store (CDS)},
  doi          = {10.24381/cds.adbb2d47},
  url          = {https://doi.org/10.24381/cds.adbb2d47},
  note         = {Accessed on 28 July 2026}
}

@misc{christensen2020rmd17,
  author       = {Christensen, Anders S. and von Lilienfeld, O. Anatole},
  title        = {{Revised MD17 dataset (rMD17)}},
  year         = {2020},
  publisher    = {figshare},
  howpublished = {Dataset},
  doi          = {10.6084/m9.figshare.12672038.v4},
  url          = {https://doi.org/10.6084/m9.figshare.12672038.v4}
}

@inproceedings{bunne2019learning,
  title={Learning generative models across incomparable spaces},
  author={Bunne, Charlotte and Alvarez-Melis, David and Krause, Andreas and Jegelka, Stefanie},
  booktitle={International conference on machine learning},
  pages={851--861},
  year={2019},
  organization={PMLR}
}

@article{seguy2017large,
  title={Large-scale optimal transport and mapping estimation},
  author={Seguy, Vivien and Damodaran, Bharath Bhushan and Flamary, R{\'e}mi and Courty, Nicolas and Rolet, Antoine and Blondel, Mathieu},
  journal={arXiv preprint arXiv:1711.02283},
  year={2017}
}


\appendix

\newpage

\section{Supplementary Methodological and Experimental Analysis}
\label{app:supplementary_analysis}

This appendix provides additional theoretical and methodological details, together with extended experimental evidence supporting the main results. We first establish the optimality of the barycentric targets used for teacher--student distillation. We then examine the proposed inductive formulations on ERA5, including teacher--student transfer, projection
selection, objective values, and computational cost. Finally, we evaluate their applicability beyond climate data through complementary experiments on rMD17-Aspirin.

\subsection{Theoretical Properties of Barycentric Distillation}
\label{app:barycentric_optimality}

This subsection establishes three theoretical properties underlying the
barycentric distillation procedure used in the proposed teacher--student
framework. First, for a fixed transport plan and latent support, the
barycentric projection is the unique sample-indexed representation that
minimizes the corresponding quadratic reconstruction cost. Second, the
projection is equivariant under Euclidean isometries of the latent support,
thereby preserving its pairwise relational geometry under global rotations,
reflections, and translations. Finally, the pointwise distillation error
controls an upper bound on the discrepancy between the pairwise Euclidean
distances induced by the student predictions and those of the barycentric
targets.

\begin{proposition}[Optimality of barycentric targets]
\label{prop:barycentric_optimality}
Let \(T\in\Pi(a,b)\) be a fixed transport plan, with \(a_i>0\)
for every \(i\), and let \(Z=[z_1,\ldots,z_m]^\top\). Then the
barycentric projection
\(
\widetilde Y
=
\operatorname{Diag}(a)^{-1}TZ,\) 
\(
\widetilde y_i
=
\frac{1}{a_i}\sum_{j=1}^{m}T_{ij}z_j,
\)
is the unique minimizer of
\(
\mathcal{Q}_T(Y;Z)
=
\sum_{i=1}^{n}\sum_{j=1}^{m}
T_{ij}\lVert y_i-z_j\rVert_2^2.
\)
Moreover, for every \(Y=[y_1,\ldots,y_n]^\top\),
\(
\mathcal{Q}_T(Y;Z)
=
\mathcal{Q}_T(\widetilde Y;Z)
+
\sum_{i=1}^{n}a_i
\lVert y_i-\widetilde y_i\rVert_2^2.
\)
\end{proposition}

\begin{proof}
Because \(T\in\Pi(a,b)\), its row sums satisfy
\(
\sum_{j=1}^{m}T_{ij}=a_i,\) 
\(i=1,\ldots,n.
\)
Since \(a_i>0\), the \(i\)-th row of the barycentric projection is
well defined and can be written as
\[
\widetilde y_i
=
\frac{1}{a_i}\sum_{j=1}^{m}T_{ij}z_j\Rightarrow
a_i\widetilde y_i
=
\sum_{j=1}^{m}T_{ij}z_j.
\]

Now fix an arbitrary \(Y=[y_1,\ldots,y_n]^\top\). For each pair
\((i,j)\), write
\[
y_i-z_j
=
(y_i-\widetilde y_i)
+
(\widetilde y_i-z_j).
\]
Expanding the squared Euclidean norm gives
\[
\begin{aligned}
\lVert y_i-z_j\rVert_2^2
=
\lVert y_i-\widetilde y_i\rVert_2^2
+
\lVert \widetilde y_i-z_j\rVert_2^2 
+
2\left\langle
y_i-\widetilde y_i,\,
\widetilde y_i-z_j
\right\rangle .
\end{aligned}
\]
Multiplying by \(T_{ij}\) and summing over \(j\), we obtain
\[
\begin{aligned}
\sum_{j=1}^{m}T_{ij}\lVert y_i-z_j\rVert_2^2
=
\left(\sum_{j=1}^{m}T_{ij}\right)
\lVert y_i-\widetilde y_i\rVert_2^2 
+
\sum_{j=1}^{m}T_{ij}
\lVert \widetilde y_i-z_j\rVert_2^2 
+
2\sum_{j=1}^{m}T_{ij}
\left\langle
y_i-\widetilde y_i,\,
\widetilde y_i-z_j
\right\rangle .
\end{aligned}
\]
The first term becomes
\[
\left(\sum_{j=1}^{m}T_{ij}\right)
\lVert y_i-\widetilde y_i\rVert_2^2
=
a_i\lVert y_i-\widetilde y_i\rVert_2^2.
\]
For the cross term, linearity of the inner product yields
\[
\begin{aligned}
\sum_{j=1}^{m}T_{ij}
\left\langle
y_i-\widetilde y_i,\,
\widetilde y_i-z_j
\right\rangle 
 =
\left\langle
y_i-\widetilde y_i,\,
\sum_{j=1}^{m}T_{ij}
(\widetilde y_i-z_j)
\right\rangle .
\end{aligned}
\]
The vector inside the second argument is zero because
\[
\begin{aligned}
\sum_{j=1}^{m}T_{ij}(\widetilde y_i-z_j)
&=
\left(\sum_{j=1}^{m}T_{ij}\right)\widetilde y_i
-
\sum_{j=1}^{m}T_{ij}z_j \\
&=
a_i\widetilde y_i-a_i\widetilde y_i
=0.
\end{aligned}
\]
Consequently,
\[
\sum_{j=1}^{m}T_{ij}\lVert y_i-z_j\rVert_2^2
=
a_i\lVert y_i-\widetilde y_i\rVert_2^2
+
\sum_{j=1}^{m}T_{ij}
\lVert \widetilde y_i-z_j\rVert_2^2.
\]
Summing this identity over \(i\) gives
\[
\mathcal{Q}_T(Y;Z)
=
\mathcal{Q}_T(\widetilde Y;Z)
+
\sum_{i=1}^{n}a_i
\lVert y_i-\widetilde y_i\rVert_2^2.
\]

Since \(a_i>0\), every term in the final sum is nonnegative.
Therefore,
\(
\mathcal{Q}_T(Y;Z)
\geq
\mathcal{Q}_T(\widetilde Y;Z).
\)
Equality holds if and only if
\(
\lVert y_i-\widetilde y_i\rVert_2^2=0\)
for every $i$,
which is equivalent to \(Y=\widetilde Y\). Hence
\(\widetilde Y\) is the unique minimizer for the fixed transport
plan \(T\).
\end{proof}

Therefore, for a fixed transport plan, the barycentric projection is the unique sample-indexed Euclidean representative minimizing the corresponding reconstruction cost. In Mean-GWMDS, this result applies to its single consensus plan. In Multi-GWMDS, it applies separately to each view-dependent plan \(T^{(v)}\), producing one unique barycentric projection per view. These view-dependent projections constitute the candidate targets used by the subsequent projection-selection procedure.

\begin{proposition}[Equivariance under Euclidean isometries]
\label{cor:isometric_equivariance}
Let \(T\in\mathbb{R}_{+}^{n\times m}\) be a fixed transport plan with
\(T\mathbf{1}_m=a\), and let
\[
\widetilde Y
=
\mathcal{B}_{T}(Z)
=
\operatorname{Diag}(a)^{-1}TZ.
\]
For any orthogonal matrix \(Q\in\mathbb{R}^{d\times d}\) and translation
vector \(c\in\mathbb{R}^{d}\), consider the isometrically transformed
latent support
\[
Z'
=
ZQ+\mathbf{1}_m c^\top.
\]
Then its barycentric projection satisfies
\[
\mathcal{B}_{T}(Z')
=
\widetilde YQ+\mathbf{1}_n c^\top.
\]
Consequently, the projected representation undergoes the same global
rotation, reflection, or translation, while all pairwise Euclidean
distances remain unchanged:
\[
D_{\mathcal{B}_{T}(Z')}
=
D_{\mathcal{B}_{T}(Z)}.
\]
\end{proposition}

\begin{proof}
Using \(T\mathbf{1}_m=a\), we obtain
\begin{align*}
\mathcal{B}_{T}(Z')
&=
\operatorname{Diag}(a)^{-1}
T\left(ZQ+\mathbf{1}_m c^\top\right) \\
&=
\operatorname{Diag}(a)^{-1}TZQ
+
\operatorname{Diag}(a)^{-1}
T\mathbf{1}_m c^\top \\
&=
\mathcal{B}_{T}(Z)Q
+
\operatorname{Diag}(a)^{-1}a\,c^\top \\
&=
\widetilde YQ+\mathbf{1}_n c^\top.
\end{align*}

Let \(\widetilde y_i\) and \(\widetilde y_j\) denote two rows of
\(\widetilde Y\), and let \(\widetilde y_i'\) and
\(\widetilde y_j'\) denote the corresponding rows of
\(\mathcal{B}_{T}(Z')\). Then
\begin{align*}
\left\|
\widetilde y_i'-\widetilde y_j'
\right\|_2^2
&=
\left\|
\left(\widetilde y_iQ+c^\top\right)
-
\left(\widetilde y_jQ+c^\top\right)
\right\|_2^2 \\
&=
\left\|
(\widetilde y_i-\widetilde y_j)Q
\right\|_2^2 \\
&=
(\widetilde y_i-\widetilde y_j)
QQ^\top
(\widetilde y_i-\widetilde y_j)^\top \\
&=
\left\|
\widetilde y_i-\widetilde y_j
\right\|_2^2,
\end{align*}
where the last equality follows from the orthogonality of \(Q\),
that is, \(QQ^\top=I\). Therefore, every pairwise Euclidean distance
is preserved, and hence
\[
D_{\mathcal{B}_{T}(Z')}
=
D_{\mathcal{B}_{T}(Z)}.
\]
\end{proof}

Proposition~\ref{cor:isometric_equivariance} clarifies the scope of the
uniqueness established in Proposition~\ref{prop:barycentric_optimality}.
Although the sample-indexed coordinates depend on the coordinate realization
selected by the teacher, globally isometric realizations yield barycentric
targets related by the same isometry and therefore preserve exactly the same
pairwise relational geometry.

\begin{proposition}[Control of relational distortion by distillation error]
\label{prop:distillation_relational_control}
Let
\(\widetilde Y=(\widetilde y_1,\ldots,\widetilde y_n)^\top\)
denote the barycentric targets and let
\(\widehat Y=(\widehat y_1,\ldots,\widehat y_n)^\top\)
be the corresponding student predictions. For weights
\(a\in\Delta_n\), define the weighted distillation error as
\[
\mathcal{L}_{\mathrm{dist}}
=
\sum_{i=1}^{n}
a_i
\left\|
\widehat y_i-\widetilde y_i
\right\|_2^2.
\]
Then the induced pairwise Euclidean distances satisfy
\[
\sum_{i,j=1}^{n}
a_i a_j
\left(
\left\|
\widehat y_i-\widehat y_j
\right\|_2
-
\left\|
\widetilde y_i-\widetilde y_j
\right\|_2
\right)^2
\leq
4\mathcal{L}_{\mathrm{dist}}.
\]
\end{proposition}
\begin{proof}
For each sample \(i\), define the student prediction error as
\[
e_i=\widehat y_i-\widetilde y_i.
\Rightarrow 
\widehat y_i=\widetilde y_i+e_i.
\]

Fix an arbitrary pair of indices \(i,j\). We first compare the
student-induced distance with the corresponding distance between the
barycentric targets. Recall that the reverse triangle inequality states
that, for any vectors \(u\) and \(v\),
\[
\bigl|\|u\|_2-\|v\|_2\bigr|
\leq
\|u-v\|_2.
\]
Applying this inequality with
\(
u=\widehat y_i-\widehat y_j\)
and
\(v=\widetilde y_i-\widetilde y_j,
\)
we obtain
\begin{align*}
\left|
\left\|
\widehat y_i-\widehat y_j
\right\|_2
-
\left\|
\widetilde y_i-\widetilde y_j
\right\|_2
\right|
\leq
\left\|
(\widehat y_i-\widehat y_j)
-
(\widetilde y_i-\widetilde y_j)
\right\|_2.
\end{align*}
The vector inside the norm can be rewritten as
\begin{align*}
(\widehat y_i-\widehat y_j)
-
(\widetilde y_i-\widetilde y_j)
&=(\widehat y_i-\widetilde y_i)
-
(\widehat y_j-\widetilde y_j)
\\
&=e_i-e_j.
\end{align*}
Therefore,
\[
\left|
\left\|
\widehat y_i-\widehat y_j
\right\|_2
-
\left\|
\widetilde y_i-\widetilde y_j
\right\|_2
\right|
\leq
\|e_i-e_j\|_2.
\]
Using the triangle inequality once more gives
\[
\|e_i-e_j\|_2
=
\|e_i+(-e_j)\|_2
\leq
\|e_i\|_2+\|e_j\|_2.
\]
Combining the preceding inequalities yields
\begin{equation}
\label{eq:pairwise-error-bound}
\left|
\left\|
\widehat y_i-\widehat y_j
\right\|_2
-
\left\|
\widetilde y_i-\widetilde y_j
\right\|_2
\right|
\leq
\|e_i\|_2+\|e_j\|_2.
\end{equation}

Squaring both sides of~\eqref{eq:pairwise-error-bound}, and using
\[
(r+s)^2
=
r^2+2rs+s^2
\leq
2r^2+2s^2,
\]
which follows from \(2rs\leq r^2+s^2\), we obtain
\begin{align*}
\left(
\left\|
\widehat y_i-\widehat y_j
\right\|_2
-
\left\|
\widetilde y_i-\widetilde y_j
\right\|_2
\right)^2
&\leq
\left(
\|e_i\|_2+\|e_j\|_2
\right)^2
\\
&\leq
2\|e_i\|_2^2+2\|e_j\|_2^2.
\end{align*}

Since \(a_i,a_j\geq 0\), multiplying both sides by \(a_i a_j\)
preserves the inequality. Summing over all pairs \(i,j\) gives
\begin{align*}
\sum_{i,j=1}^{n}
a_i a_j
\left(
\left\|
\widehat y_i-\widehat y_j
\right\|_2
-
\left\|
\widetilde y_i-\widetilde y_j
\right\|_2
\right)^2
\leq
2\sum_{i,j=1}^{n}
a_i a_j\|e_i\|_2^2
+
2\sum_{i,j=1}^{n}
a_i a_j\|e_j\|_2^2.
\end{align*}

We now simplify the two terms on the right-hand side. In the first
term, \(\|e_i\|_2^2\) does not depend on \(j\). Hence,
\begin{align*}
2\sum_{i,j=1}^{n}
a_i a_j\|e_i\|_2^2
&=
2\sum_{i=1}^{n}
a_i\|e_i\|_2^2
\left(
\sum_{j=1}^{n}a_j
\right) \\
&=
2\sum_{i=1}^{n}
a_i\|e_i\|_2^2,
\end{align*}
where we used \(\sum_{j=1}^{n}a_j=1\). Similarly, in the second
term, \(\|e_j\|_2^2\) does not depend on \(i\), so
\begin{align*}
2\sum_{i,j=1}^{n}
a_i a_j\|e_j\|_2^2
&=
2\left(
\sum_{i=1}^{n}a_i
\right)
\sum_{j=1}^{n}
a_j\|e_j\|_2^2 \\
&=
2\sum_{j=1}^{n}
a_j\|e_j\|_2^2.
\end{align*}

The index used in a finite sum is arbitrary. Therefore,
\[
\sum_{j=1}^{n}a_j\|e_j\|_2^2
=
\sum_{i=1}^{n}a_i\|e_i\|_2^2.
\]
Consequently,
\begin{align*}
\sum_{i,j=1}^{n}
a_i a_j
\left(
\left\|
\widehat y_i-\widehat y_j
\right\|_2
-
\left\|
\widetilde y_i-\widetilde y_j
\right\|_2
\right)^2
\leq
4\sum_{i=1}^{n}
a_i\|e_i\|_2^2.
\end{align*}
Finally, by the definition of the weighted distillation loss,
\[
\mathcal{L}_{\mathrm{dist}}
=
\sum_{i=1}^{n}
a_i
\left\|
\widehat y_i-\widetilde y_i
\right\|_2^2
=
\sum_{i=1}^{n}a_i\|e_i\|_2^2.
\]
Thus,
\[
\sum_{i,j=1}^{n}
a_i a_j
\left(
\left\|
\widehat y_i-\widehat y_j
\right\|_2
-
\left\|
\widetilde y_i-\widetilde y_j
\right\|_2
\right)^2
\leq
4\mathcal{L}_{\mathrm{dist}},
\]
which proves the result.
\end{proof}

For the uniform empirical measures considered in this work,
\(a_i=1/n\), Proposition~\ref{prop:distillation_relational_control}
specializes to
\[
\frac{1}{n^2}
\left\|
D_{\widehat Y}-D_{\widetilde Y}
\right\|_F^2
\leq
\frac{4}{n}
\left\|
\widehat Y-\widetilde Y
\right\|_F^2.
\]
Therefore, the mean-squared pointwise distillation error controls an
upper bound on the average discrepancy between the pairwise Euclidean
distance matrices induced by the student predictions and the barycentric
targets.

\subsection{Algorithmic Summary}

\begin{algorithm}[h]
\caption{Multi-view barycentric distillation}
\label{alg:multiview_barycentric_distillation}
\begin{algorithmic}[1]
\Require Corresponding training views
\(\{X^{(v)}\}_{v=1}^{V}\), view weights
\(\{\lambda_v\}_{v=1}^{V}\), embedding dimension \(d\),
teacher type \(q\in\{\mathrm{Mean},\mathrm{Multi}\}\), and
aggregation criterion \(\operatorname{Agg}\)
\Ensure Inductive mapping \(f_{\theta^\star}\)

\State Fit and apply all preprocessing transformations using training data only
\For{\(v=1,\ldots,V\)}
    \State Construct and normalize the relational matrix \(D_X^{(v)}\)
\EndFor

\If{\(q=\mathrm{Mean}\)}
    \State \(\overline D_X
        \gets \sum_{v=1}^{V}\lambda_v D_X^{(v)}\)
    \State \((Z^\star,T^\star)
        \gets \Call{GW\text{-}MDS}{\overline D_X,d}\)
    \State \(\widetilde Y
        \gets \mathcal B_{T^\star}(Z^\star)\)
\Else
    \State \((Z^\star,\{T^{(v)\star}\}_{v=1}^{V})
        \gets
        \Call{Multi\text{-}GWMDS}
        {\{D_X^{(v)}\}_{v=1}^{V},
         \{\lambda_v\}_{v=1}^{V},d}\)

    \For{\(v=1,\ldots,V\)}
        \State \(\widetilde Y^{(v)}
            \gets \mathcal B_{T^{(v)\star}}(Z^\star)\)
        \State \(s_v
            \gets
            \operatorname{Agg}_{u=1,\ldots,V}
            \rho_{\mathrm P}
            \left(
                D_{\widetilde Y^{(v)}},
                D_X^{(u)}
            \right)\)
    \EndFor

    \State \(v^\star
        \gets
        \arg\max_{v\in\{1,\ldots,V\}}s_v\)
    \State \(\widetilde Y
        \gets \widetilde Y^{(v^\star)}\)
\EndIf

\State Split the training indices into fitting and validation subsets
\State Standardize \(\widetilde Y\) using statistics from the fitting subset
\State Train \(f_\theta\) on the fitting subset by minimizing
\[
    \mathcal L_{\mathrm{distill}}(\theta)
    =
    \sum_{i\in\mathcal I_{\mathrm{fit}}}
    \left\|
        f_\theta\!\left(
            x_i^{(1)},\ldots,x_i^{(V)}
        \right)
        -
        \widetilde y_i
    \right\|_2^2
\]
\State Set \(\theta^\star\) to the checkpoint with the lowest validation loss
\State \Return \(f_{\theta^\star}\)

\end{algorithmic}
\end{algorithm}

When the Multi-GWMDS teacher produces multiple view-dependent
barycentric projections and a single representation is required, the
candidate projections are evaluated using the Pearson correlation
between their induced distance matrices and the relational structures
of the training views. The resulting view-wise correlations are
combined using a prespecified aggregation criterion:
\begin{equation}
    s_v
    =
    \operatorname{Agg}_{u=1,\ldots,V}
    \rho_{\mathrm{P}}
    \left(
        D_{\widetilde Y^{(v)}},
        D_X^{(u)}
    \right),
    \qquad
    v^\star
    =
    \argmax{v\in\{1,\ldots,V\}} s_v,
    \label{eq:projection_selection}
\end{equation}
where \(\rho_{\mathrm{P}}\) denotes the Pearson correlation computed
between the off-diagonal entries of the two distance matrices, and
\(\operatorname{Agg}\) may be the arithmetic mean, the median, or the
minimum. Choosing the minimum yields a maximin selection rule, since
the selected candidate maximizes its worst agreement across the views.
In our experiments, \(\operatorname{Agg}\) is instantiated as the
arithmetic mean.

For \(V=1\), Algorithm~\ref{alg:multiview_barycentric_distillation}
recovers the single-view barycentric distillation formulation. At
inference time, the preprocessing transformations fitted on the
training data are first applied to the unseen observation. Its
embedding is then obtained directly through a forward pass of the
trained student, without constructing test relational matrices,
estimating test-set couplings, or solving additional GW problems.

\subsection{Additional Experimental Analysis on ERA5}
\label{app:era5_analysis}

\paragraph{Teacher--student transfer.}
Table~\ref{tab:era5_teacher_student_transfer} confirms that barycentric
distillation closely reproduces the relational behavior of both
teachers on the training observations. For Euclidean and geodesic
relations, the weighted Pearson correlation of the Mean-GWMDS student
differs from that of its teacher by at most \(0.0002\). The largest
difference occurs under cosine dissimilarity, for which the correlation
decreases from \(0.5175\) to \(0.4960\), while Spearman correlation,
trustworthiness, and stress remain close to their teacher values.
The selected-projection student is similarly faithful: across all
geometries and metrics, its largest difference from the corresponding
Multi-GWMDS teacher is \(0.0081\).

The test-set results remain comparable to the corresponding training
values and exhibit no systematic deterioration in the global
correlations or stress. The occasionally higher test correlations
should not be interpreted as an improvement over the training
representations, since the metrics are computed from different sets of
pairwise relations: \(386\) training locations and \(97\) held-out
locations. Nevertheless, their consistency indicates that the student
does not merely memorize the barycentric targets and can transfer their
relational organization to unseen locations.

\begin{table*}[t]
\centering
\caption{Teacher-to-student transfer on ERA5. All metrics are uniformly
averaged over the four views. Teacher representations are available only
for the training observations.}
\label{tab:era5_teacher_student_transfer}
\small
\setlength{\tabcolsep}{4.2pt}
\begin{tabular}{lllcccc}
\toprule
Geometry & Representation & Split
& \(r\uparrow\) & \(\rho\uparrow\)
& Trust.\(\uparrow\) & Stress\(\downarrow\) \\
\midrule
Euclidean
& Mean-GWMDS teacher & Train
& 0.7661 & 0.7640 & 0.9065 & 0.2914 \\
& Ind. Mean-GWMDS & Train
& 0.7661 & 0.7638 & 0.9050 & 0.2915 \\
& Ind. Mean-GWMDS & Test
& 0.7853 & 0.7958 & 0.8975 & 0.2798 \\
& Multi-GWMDS teacher, selected & Train
& 0.7009 & 0.6977 & 0.8617 & 0.3388 \\
& Ind. Multi-GWMDS & Train
& 0.7070 & 0.7025 & 0.8653 & 0.3357 \\
& Ind. Multi-GWMDS & Test
& 0.7470 & 0.7572 & 0.8643 & 0.3137 \\
\midrule
Geodesic
& Mean-GWMDS teacher & Train
& 0.8139 & 0.8012 & 0.9531 & 0.2761 \\
& Ind. Mean-GWMDS & Train
& 0.8137 & 0.8010 & 0.9531 & 0.2762 \\
& Ind. Mean-GWMDS & Test
& 0.8507 & 0.8439 & 0.9373 & 0.2603 \\
& Multi-GWMDS teacher, selected & Train
& 0.7357 & 0.7233 & 0.9105 & 0.3334 \\
& Ind. Multi-GWMDS & Train
& 0.7423 & 0.7314 & 0.9151 & 0.3306 \\
& Ind. Multi-GWMDS & Test
& 0.7855 & 0.7826 & 0.9049 & 0.3106 \\
\midrule
Cosine
& Mean-GWMDS teacher & Train
& 0.5175 & 0.6326 & 0.9129 & 0.5743 \\
& Ind. Mean-GWMDS & Train
& 0.4960 & 0.6265 & 0.9121 & 0.5757 \\
& Ind. Mean-GWMDS & Test
& 0.5394 & 0.6582 & 0.8736 & 0.5641 \\
& Multi-GWMDS teacher, selected & Train
& 0.4550 & 0.5539 & 0.8515 & 0.5878 \\
& Ind. Multi-GWMDS & Train
& 0.4594 & 0.5573 & 0.8513 & 0.5869 \\
& Ind. Multi-GWMDS & Test
& 0.4830 & 0.5875 & 0.8322 & 0.5712 \\
\bottomrule
\end{tabular}
\end{table*}

\begin{table*}[t]
\centering
\caption{View-wise out-of-sample correlations on ERA5. Each entry reports
Pearson/Spearman correlation, \(r/\rho\).}
\label{tab:era5_viewwise_correlations}
\small
\setlength{\tabcolsep}{4pt}
\begin{tabular}{llcccc}
\toprule
Geometry & Method
& Temperature & Dewpoint & Pressure & Precipitation \\
\midrule
Euclidean
& Ind. Mean-GWMDS
& 0.8995/0.9057 & 0.9232/0.9186
& 0.8981/0.8811 & 0.4203/0.4777 \\
& Ind. Multi-GWMDS
& 0.9210/0.9291 & 0.8239/0.8187
& 0.8444/0.8400 & 0.3988/0.4411 \\
\midrule
Geodesic
& Ind. Mean-GWMDS
& 0.9194/0.9290 & 0.9045/0.9033
& 0.8767/0.8565 & 0.7021/0.6867 \\
& Ind. Multi-GWMDS
& 0.9393/0.9313 & 0.7821/0.7950
& 0.7772/0.7626 & 0.6436/0.6416 \\
\midrule
Cosine
& Ind. Mean-GWMDS
& 0.7575/0.7770 & 0.5337/0.5589
& 0.2667/0.6706 & 0.5995/0.6263 \\
& Ind. Multi-GWMDS
& 0.8967/0.8935 & 0.2920/0.3448
& 0.2612/0.6197 & 0.4821/0.4919 \\
\bottomrule
\end{tabular}
\end{table*}

\begin{table}[t]
\centering
\caption{Training-only selection scores for the four Multi-GWMDS
barycentric projections. Each score is the mean Pearson correlation
across the four ERA5 relational matrices.}
\label{tab:era5_projection_selection}
\small
\begin{tabular}{lccc}
\toprule
Projection & Euclidean & Geodesic & Cosine \\
\midrule
View 1: temperature
& \textbf{0.7009} & \textbf{0.7357} & \textbf{0.4550} \\
View 2: dewpoint
& 0.6948 & 0.6995 & 0.3762 \\
View 3: pressure
& 0.6570 & 0.6780 & 0.4010 \\
View 4: precipitation
& 0.5626 & 0.6453 & 0.2993 \\
\bottomrule
\end{tabular}
\end{table}

\begin{table}[t]
\centering
\caption{Last recorded training objective after \(100\) outer
iterations and out-of-sample Pearson correlation. The objective values
are uniformly averaged over the four views.}
\label{tab:era5_objective_mismatch}
\small
\begin{tabular}{lcccc}
\toprule
& \multicolumn{2}{c}{GW objective}
& \multicolumn{2}{c}{Test \(r\)} \\
\cmidrule(lr){2-3}
\cmidrule(lr){4-5}
Geometry & Teacher & Direct
& Ind. Multi & Direct \\
\midrule
Euclidean & 0.01552 & 0.01565 & 0.7470 & 0.4576 \\
Geodesic  & 0.01176 & 0.01169 & 0.7855 & 0.3409 \\
Cosine    & 0.10290 & 0.10328 & 0.4830 & 0.4416 \\
\bottomrule
\end{tabular}
\end{table}

\paragraph{View-wise behavior.}
Table~\ref{tab:era5_viewwise_correlations} shows that the advantage of geodesic relations is particularly pronounced for total
precipitation. For Inductive Mean-GWMDS, its Pearson correlation
increases from \(0.4203\) under Euclidean relations to \(0.7021\) under geodesic relations. This improvement explains a substantial part
of the superior view-averaged geodesic result. Under cosine
dissimilarity, surface pressure instead exhibits a large discrepancy
between Pearson and Spearman correlation, \(0.2667\) versus \(0.6706\).
Thus, its relational ordering is partially preserved, whereas the
magnitudes of its pairwise dissimilarities are poorly reproduced.

The selected Multi-GWMDS projection remains more strongly associated
with temperature. Under geodesic relations, it improves the
temperature correlation from \(0.9194\) to \(0.9393\) relative to
Mean-GWMDS, but decreases the correlations for dewpoint temperature,
surface pressure, and precipitation. Mean-GWMDS therefore provides a
more balanced consensus by accepting a small reduction in
temperature-view fidelity.

\paragraph{Projection selection and objective ambiguity.}
As reported in Table~\ref{tab:era5_projection_selection}, the
temperature-induced projection is selected for all three relational
geometries. Its Euclidean score exceeds that of the dewpoint projection
by only \(0.0061\), however, whereas the corresponding margins are
\(0.0362\) and \(0.0540\) for geodesic and cosine relations.

Finally, Table~\ref{tab:era5_objective_mismatch} shows that the final
objectives of the Multi-GWMDS teacher and direct neural GW differ by
less than \(1\%\) under every geometry. For geodesic relations,
the direct model even attains a slightly lower objective while producing
a substantially smaller test correlation. The result confirms that GW agreement under an
optimized coupling does not determine the sample-indexed
correspondence required for prediction.

\subsubsection{Computational Cost}
\label{app:computational_cost}

\begin{table}[h]
\centering
\caption{ERA5 execution times in seconds for a single run with seed
\(0\). Timings are implementation- and hardware-dependent.}
\label{tab:era5_runtime}
\small
\begin{tabular}{lrrrr}
\toprule
Stage & Euclidean & Geodesic & Cosine & Mean \\
\midrule
Mean-GWMDS teacher
& 76.45 & 80.58 & 82.00 & 79.68 \\
Multi-GWMDS teacher
& 238.00 & 229.60 & 244.23 & 237.28 \\
Mean student
& 1.97 & 2.18 & 1.34 & 1.83 \\
Selected-projection student
& 1.39 & 1.23 & 1.41 & 1.34 \\
Direct multi-view GW
& 193.72 & 192.84 & 240.32 & 208.96 \\
\bottomrule
\end{tabular}
\end{table}

Let \(n\) denote the number of training observations, \(V\) the number
of views, \(K_T\) the number of teacher iterations, and
\(\mathcal C_{\mathrm{GW}}(n)\) the cost of one dense GW-plan update,
including its conditional-gradient and linear optimal-transport
subproblems. Mean-GWMDS performs one such update per outer iteration,
leading to the dominant cost
\[
    \mathcal{O}\left(K_T\mathcal{C}_{\mathrm{GW}}(n)\right),
\]
whereas Multi-GWMDS performs one update for each view:
\[
    \mathcal O\!\left(
        V K_T\mathcal C_{\mathrm{GW}}(n)
    \right).
\]
The final barycentric projections require
\(\mathcal O(Vn^2d)\) operations. Storing the relational matrices and
transport plans requires \(\mathcal O(Vn^2)\) memory, which constitutes
the principal scalability limitation of the teacher stage.

Let \(P_\theta\) denote the effective cost of a forward and backward
pass through the student network. Full-batch distillation over \(K_S\)
epochs has cost approximately
\(\mathcal O(K_S nP_\theta)\), without additional GW problems.
After training, embedding \(m\) unseen observations requires only
\(\mathcal O(mP_\theta)\) operations and does not require storing an
\(m\times m\) relational matrix.

The timings in Table~\ref{tab:era5_runtime} agree with this analysis.
Student distillation takes less than \(3\%\) of its teacher's runtime.
The complete Inductive Mean-GWMDS pipeline is approximately
\(2.3\)--\(2.9\) times faster than direct multi-view GW in these
experiments. Multi-GWMDS and direct multi-view GW have comparable
training times because both repeatedly solve one GW problem per view.
The advantage of distillation therefore lies primarily in the
reusability of the teacher targets and in inexpensive out-of-sample
inference.

We compute the unregularized squared-loss GW transport plans using
the conditional-gradient solver provided by the POT library
\citep{flamary2021pot}. For this solver,
\(\mathcal C_{\mathrm{GW}}(n)=\mathcal O(K_{\mathrm{GW}}n^3)\), where
\(K_{\mathrm{GW}}\leq 100\) is the number of inner conditional-gradient
iterations; hence, for a fixed iteration budget, GW computation scales
cubically with \(n\)~\citep{peyre2016gromov}.

\subsubsection{Qualitative Results and Optimization Behavior}
\label{app:qualitative_analysis}

\begin{figure*}[h]
    \centering

    \begin{subfigure}[t]{0.32\textwidth}
        \centering
        \includegraphics[width=\linewidth]{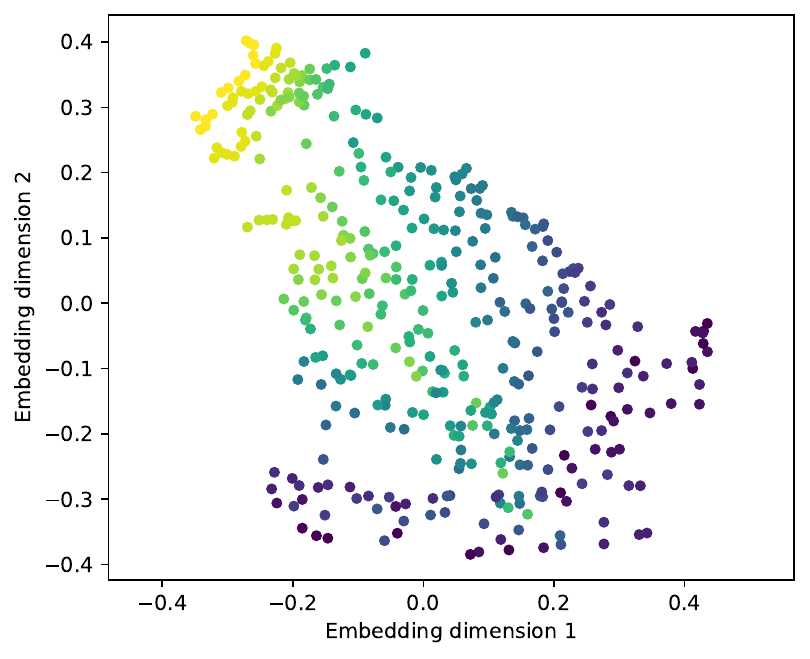}
        \caption{Mean-GWMDS consensus.}
        \label{fig:era5_mean_teacher}
    \end{subfigure}
    \hfill
    \begin{subfigure}[t]{0.32\textwidth}
        \centering
        \includegraphics[width=\linewidth]{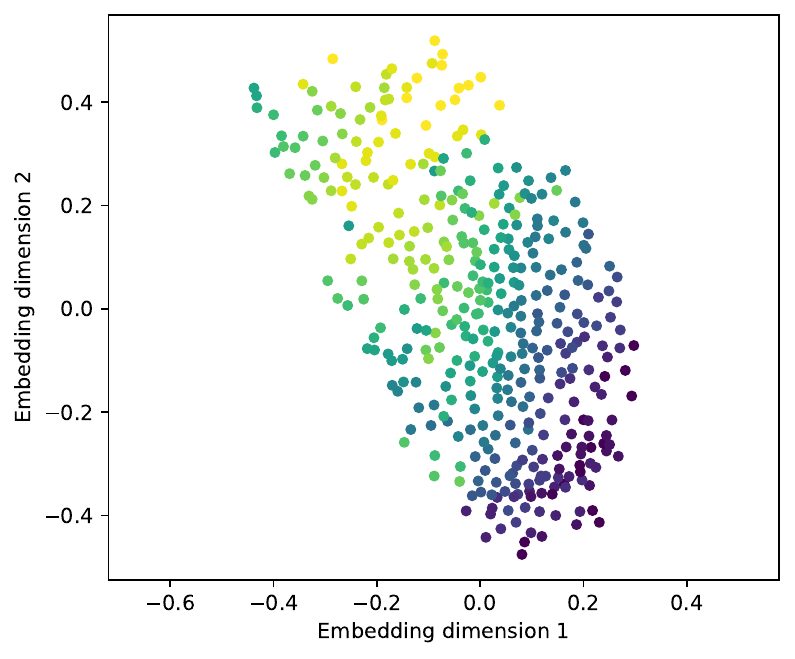}
        \caption{Temperature projection.}
        \label{fig:era5_multi_projection_v1}
    \end{subfigure}
    \hfill
    \begin{subfigure}[t]{0.32\textwidth}
        \centering
        \includegraphics[width=\linewidth]{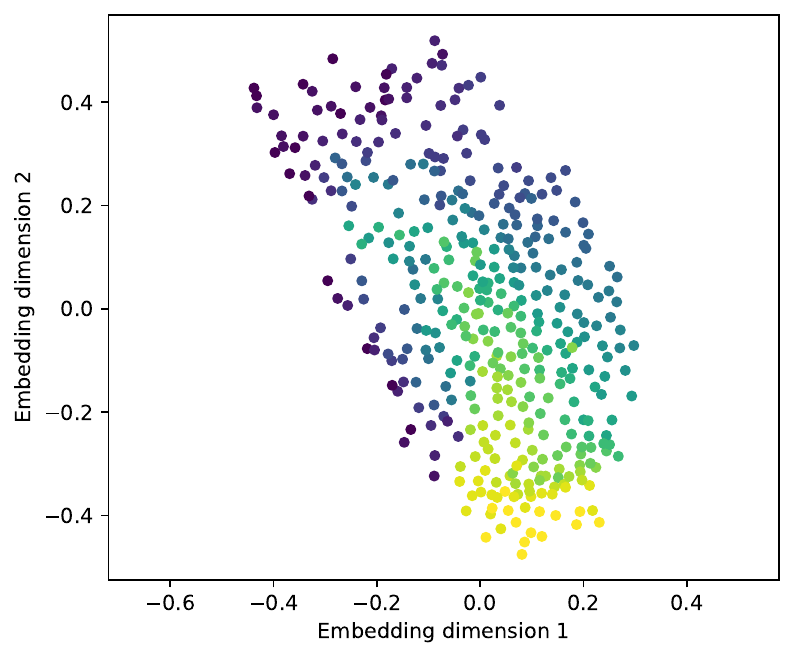}
        \caption{Dewpoint projection.}
        \label{fig:era5_multi_projection_v2}
    \end{subfigure}

    \vspace{0.8em}

    \begin{subfigure}[t]{0.32\textwidth}
        \centering
        \includegraphics[width=\linewidth]{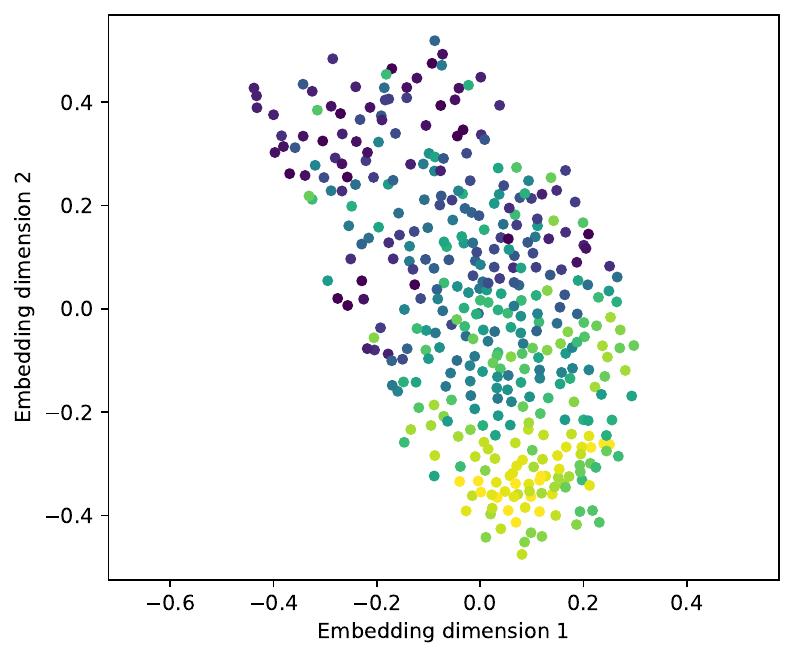}
        \caption{Surface-pressure projection.}
        \label{fig:era5_multi_projection_v3}
    \end{subfigure}
    \hspace{0.025\textwidth}
    \begin{subfigure}[t]{0.32\textwidth}
        \centering
        \includegraphics[width=\linewidth]{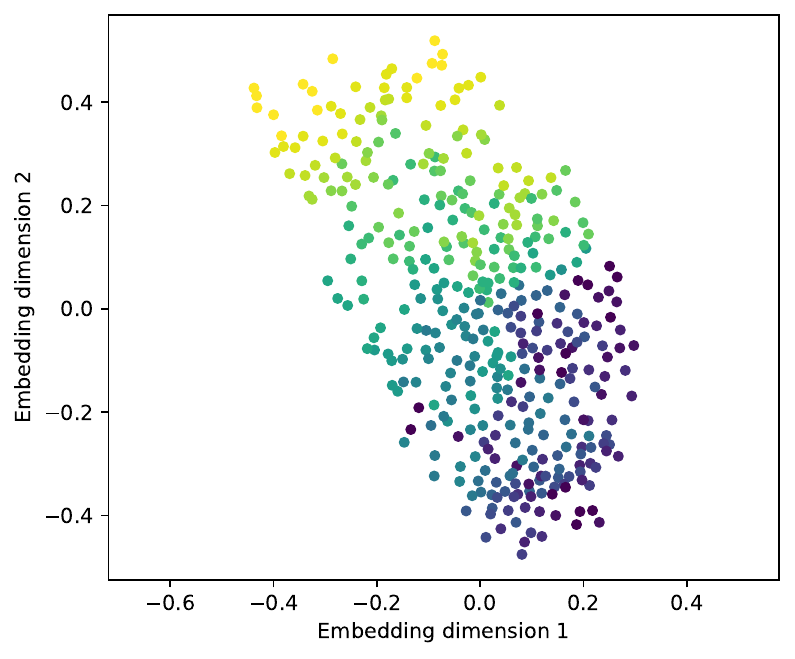}
        \caption{Total-precipitation projection.}
        \label{fig:era5_multi_projection_v4}
    \end{subfigure}

    \caption{Geodesic teacher representations for ERA5. }
    \label{fig:era5_geodesic_teachers}
\end{figure*}

We complement the quantitative results with a qualitative examination of
the teacher and out-of-sample representations, followed by an empirical
analysis of the optimization behavior of the proposed methods and the
direct neural GW baseline.

\paragraph{Teacher representations and barycentric projections.}
Figure~\ref{fig:era5_geodesic_teachers} illustrates the representations
obtained from the geodesic ERA5 relations. The Mean-GWMDS
barycentric target exhibits a consensus organization that combines the
relational information provided by the four meteorological variables. In
contrast, the Multi-GWMDS projections are generated from the same latent
support but through different view-dependent transport plans. They
therefore differ in their sample-indexed organization, demonstrating that
a shared latent geometry does not imply a unique correspondence with the original observations.

The temperature-induced projection provides the highest average
training correlation and is consequently selected for distillation.
Nevertheless, this projection favors the relational structure of
temperature, whereas the Mean-GWMDS target provides a more balanced
compromise across temperature, dewpoint temperature, surface pressure,
and total precipitation. Differences in global orientation should not be
interpreted as errors, since GW-based objectives are invariant to
rotations and reflections. The relevant distinction lies in the local
organization of the observations and in the correspondence between latent points and sample indices.

\paragraph{Out-of-sample representations.}

\begin{figure*}[h]
    \centering

    \begin{subfigure}[t]{0.48\textwidth}
        \centering
        \includegraphics[width=\linewidth]{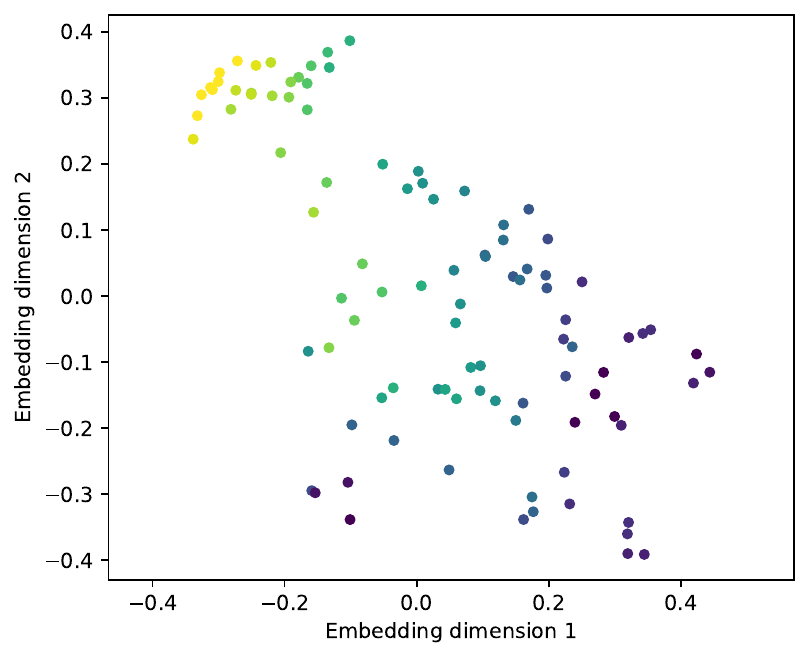}
        \caption{Inductive Mean-GWMDS.}
        \label{fig:era5_test_mean_gwmds}
    \end{subfigure}
    \hfill
    \begin{subfigure}[t]{0.48\textwidth}
        \centering
        \includegraphics[width=\linewidth]{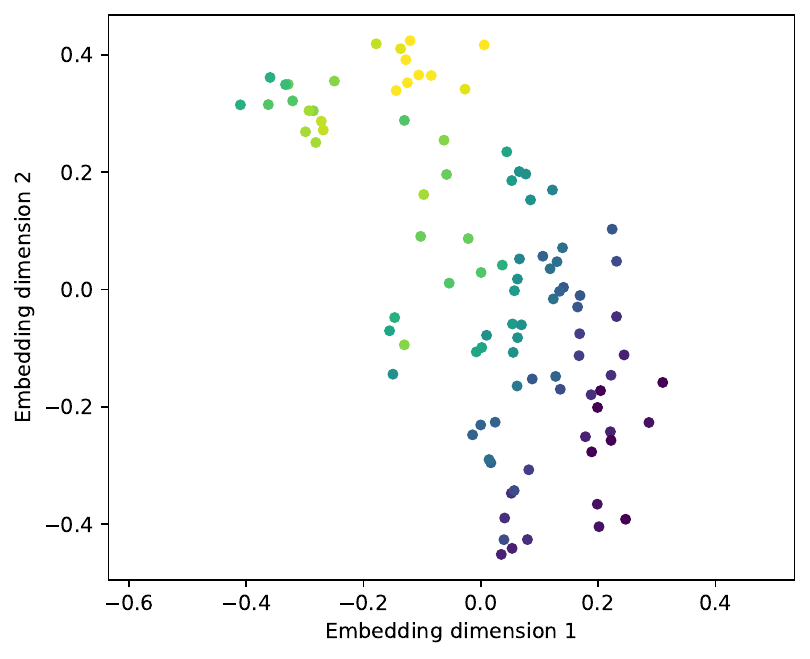}
        \caption{Inductive Multi-GWMDS.}
        \label{fig:era5_test_multi_gwmds}
    \end{subfigure}

    \vspace{0.8em}

    \begin{subfigure}[t]{0.48\textwidth}
        \centering
        \includegraphics[width=\linewidth]{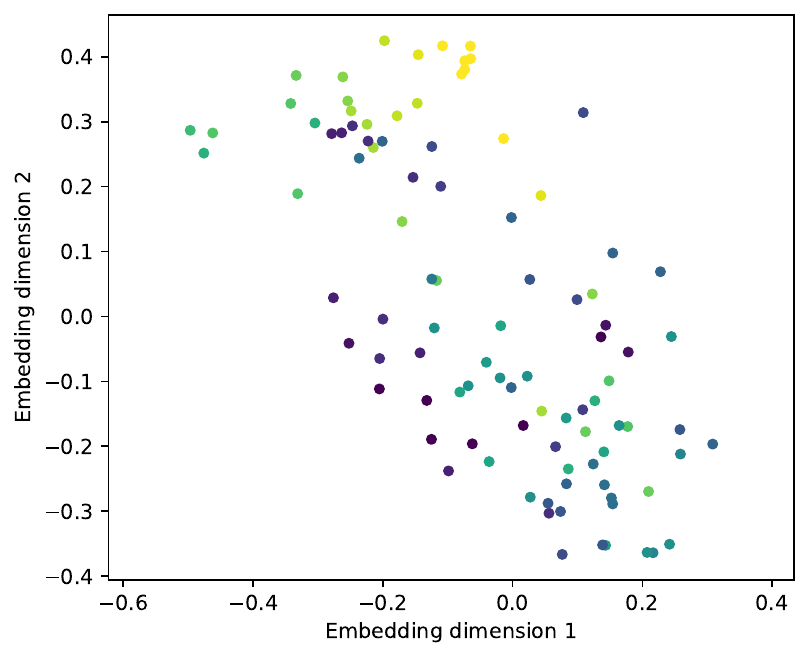}
        \caption{Direct multi-view GW.}
        \label{fig:era5_test_direct_gw}
    \end{subfigure}
    \hfill
    \begin{subfigure}[t]{0.48\textwidth}
        \centering
        \includegraphics[width=\linewidth]{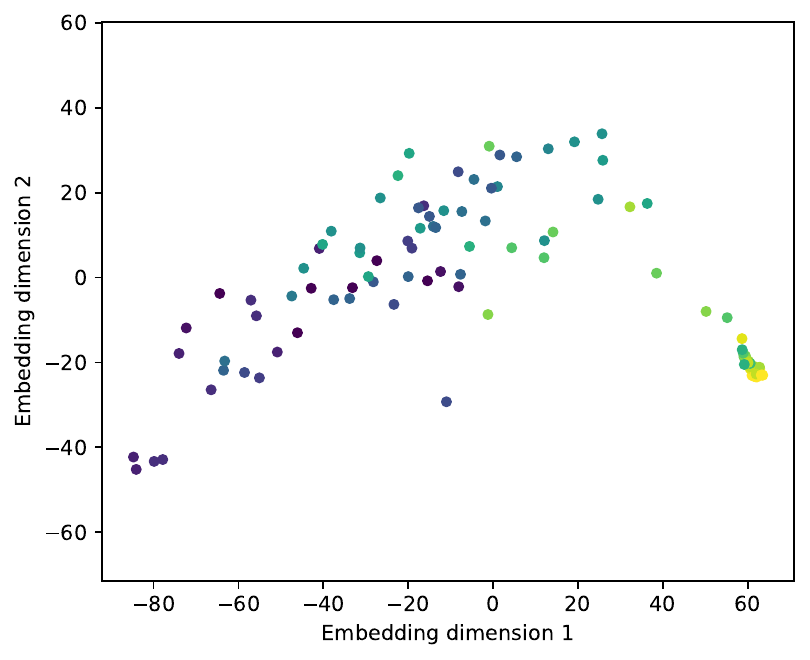}
        \caption{Concatenated PCA.}
        \label{fig:era5_test_pca}
    \end{subfigure}

    \caption{Out-of-sample embeddings of the held-out ERA5 locations
    under geodesic relations.}
    \label{fig:era5_geodesic_test}
\end{figure*}

Figure~\ref{fig:era5_geodesic_test} compares the representations produced
for the held-out locations. Both distilled models generate coherent
out-of-sample embeddings without constructing a test-set transport plan.
Inductive Mean-GWMDS retains the consensus structure of its teacher and
provides the most balanced preservation across the four views. Inductive
Multi-GWMDS instead reproduces the selected temperature-induced
projection, yielding stronger preservation of the temperature geometry
but weaker agreement with the remaining views.

The direct neural GW model produces a substantially less informative
sample-indexed organization, despite attaining a final GW objective close
to that of the Multi-GWMDS teacher. This visual discrepancy reinforces
the quantitative results in
Table~\ref{tab:era5_objective_mismatch}: optimizing structural agreement
under a freely estimated coupling does not ensure that the network output
is aligned with the original sample identities. The PCA baseline captures
part of the dominant global variation, but it does not explicitly balance
the distinct relational geometries encoded by the four views.

\paragraph{Optimization behavior.}

\begin{figure*}[t]
    \centering

    \begin{subfigure}[t]{0.32\textwidth}
        \centering
        \includegraphics[width=\linewidth]{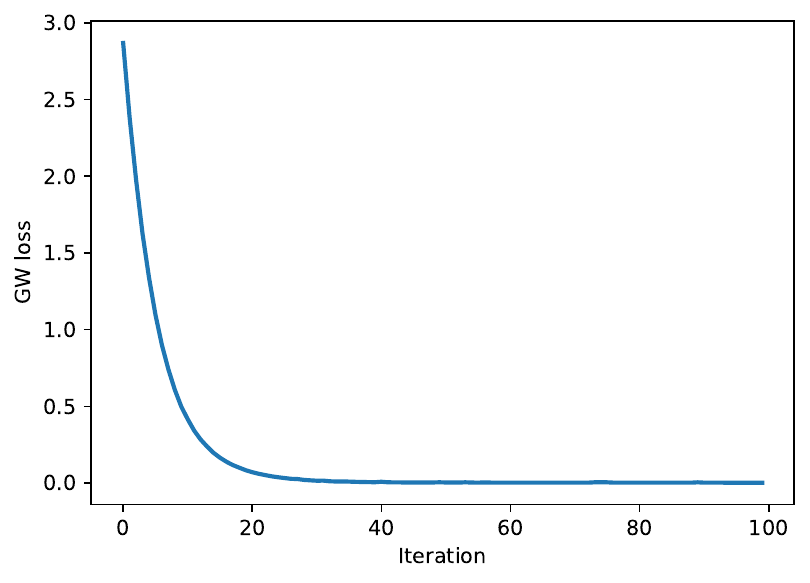}
        \caption{Mean-GWMDS teacher.}
        \label{fig:era5_mean_teacher_convergence}
    \end{subfigure}
    \hfill
    \begin{subfigure}[t]{0.32\textwidth}
        \centering
        \includegraphics[width=\linewidth]{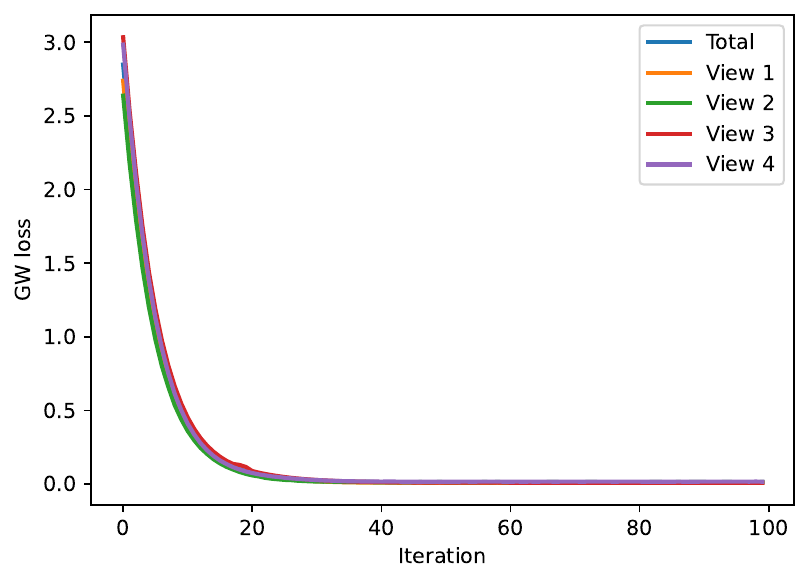}
        \caption{Multi-GWMDS teacher.}
        \label{fig:era5_multi_teacher_convergence}
    \end{subfigure}
    \hfill
    \begin{subfigure}[t]{0.32\textwidth}
        \centering
        \includegraphics[width=\linewidth]{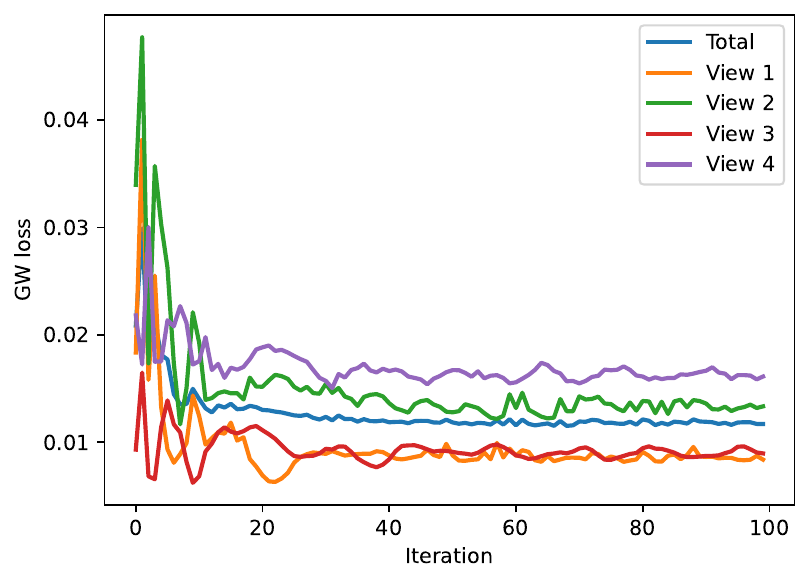}
        \caption{Direct multi-view GW.}
        \label{fig:era5_direct_gw_convergence}
    \end{subfigure}

    \vspace{0.8em}

    \begin{subfigure}[t]{0.40\textwidth}
        \centering
        \includegraphics[width=\linewidth]{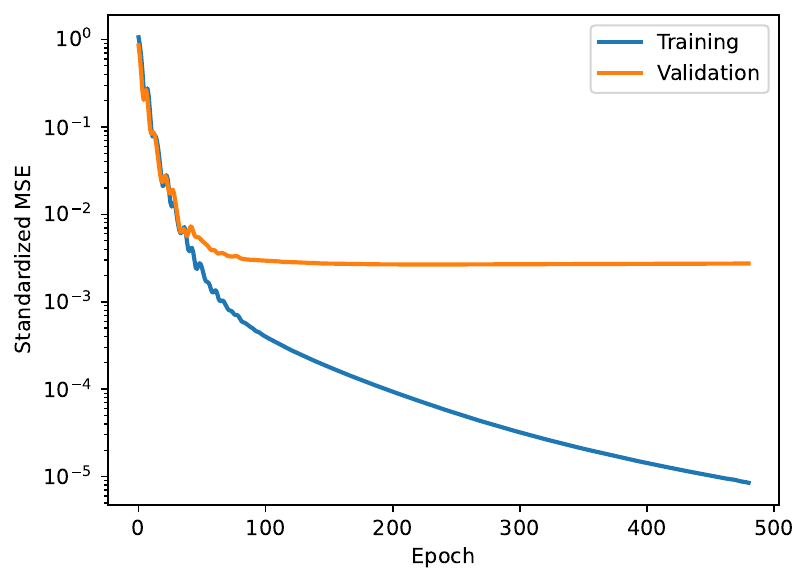}
        \caption{Inductive Mean-GWMDS distillation.}
        \label{fig:era5_mean_student_convergence}
    \end{subfigure}
    \hspace{0.04\textwidth}
    \begin{subfigure}[t]{0.40\textwidth}
        \centering
        \includegraphics[width=\linewidth]{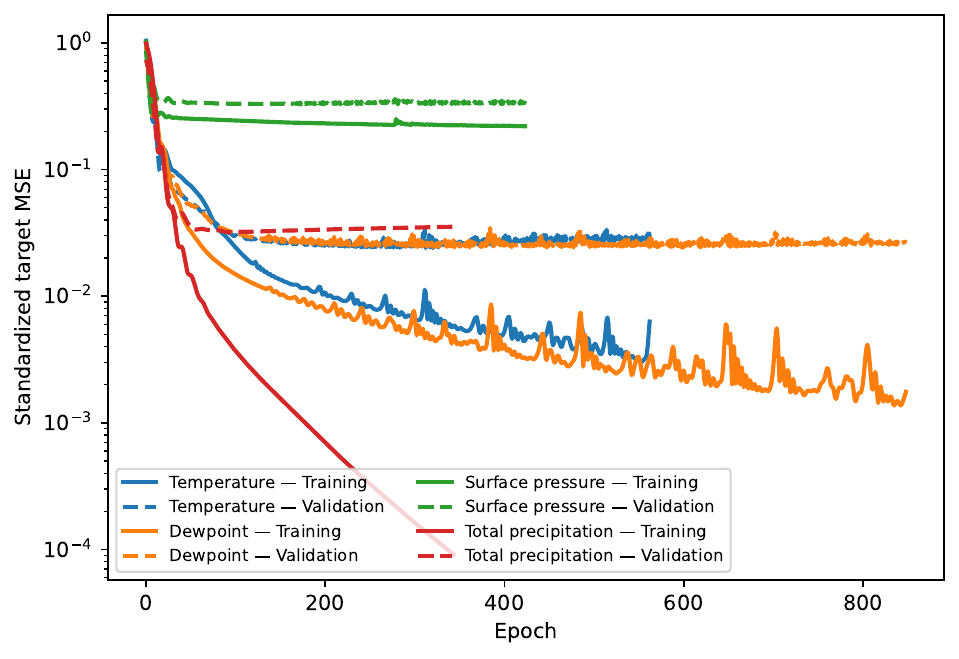}
        \caption{Multi-GWMDS projection distillation.}
        \label{fig:era5_multi_students_convergence}
    \end{subfigure}

    \caption{Optimization trajectories for the geodesic ERA5
    experiment.}
    \label{fig:era5_geodesic_convergence}
\end{figure*}

Figure~\ref{fig:era5_geodesic_convergence} reports the optimization trajectories of the geodesic models. The Mean-GWMDS and Multi-GWMDS teacher objectives decrease and stabilize over the outer iterations, indicating that the barycentric targets are extracted from converged relational representations. Multi-GWMDS is more expensive because each outer iteration requires a separate GW-plan update for every view.

The student losses decrease rapidly and smoothly, showing that the
sample-indexed barycentric targets can be accurately approximated by the
neural mapping. The direct neural GW objective also decreases to a value
comparable to that of the transductive teacher. Its convergence therefore
does not resolve the correspondence ambiguity; a small GW objective may
coexist with poor sample-indexed preservation. Taken together, the curves
indicate that the performance difference is caused by the supervision
provided by the barycentric targets rather than by a failure to optimize
the direct objective.

\subsection{Additional Experimental Analysis on rMD17-Aspirin}
\label{app:rmd17_analysis}



\paragraph{Teacher--student transfer.}
The distilled students closely reproduce their corresponding
barycentric targets on the training data. Across the three relational
geometries, the difference in view-averaged Pearson correlation between
the Mean-GWMDS teacher and the Inductive Mean-GWMDS student ranges from
\(0.0148\) to \(0.0197\). For the student associated with the projection
selected by Multi-GWMDS, the largest difference from its teacher is
\(0.0045\). These results show that the neural mappings accurately
approximate both the consensus target and the selected view-dependent
projection.

The molecular experiment is nevertheless more challenging than ERA5
from an out-of-sample perspective. Although the training losses approach
zero, the validation losses stabilize at higher values, with the best
validation losses ranging approximately from \(0.105\) to \(0.214\).
This separation indicates that interpolating the barycentric
organization of unseen molecular conformations is more difficult than
fitting the training targets. Validation-based early stopping is
therefore particularly important for this experiment, especially for
students trained exclusively from the force-Gram projection.

\paragraph{Out-of-sample performance.}
The view-averaged results are reported in
Table~\ref{tab:multiview_inductive_results}. We next complement those
results by examining teacher--student transfer, view-wise behavior,
projection selection, GW objective values, and computational cost.

\paragraph{View-wise behavior.}
The view-wise results in
Table~\ref{tab:rmd17_viewwise_correlations} reveal a substantial
imbalance between the two molecular representations.

\begin{table*}[t]
    \centering
    \caption{View-wise test Pearson correlations on rMD17-Aspirin.
    The reported embeddings are produced by the two inductive
    teacher--student formulations.}
    \label{tab:rmd17_viewwise_correlations}
    \small
    \begin{tabular}{llcc}
        \toprule
        Geometry & Method
        & Interatomic-distance view
        & Force-Gram view \\
        \midrule

        Euclidean
        & Inductive Mean-GWMDS
        & 0.7294 & \textbf{0.2178} \\

        & Inductive Multi-GWMDS
        & \textbf{0.7506} & 0.0624 \\
        \midrule

        Geodesic
        & Inductive Mean-GWMDS
        & 0.7860 & \textbf{0.1715} \\

        & Inductive Multi-GWMDS
        & \textbf{0.7928} & 0.0524 \\
        \midrule

        Cosine
        & Inductive Mean-GWMDS
        & 0.7740 & \textbf{0.2743} \\

        & Inductive Multi-GWMDS
        & \textbf{0.8204} & 0.0669 \\
        \bottomrule
    \end{tabular}
\end{table*}

For Inductive Mean-GWMDS, the test Pearson correlations with the
interatomic-distance view are \(0.7294\), \(0.7860\), and \(0.7740\)
under Euclidean, geodesic, and cosine relations, respectively.
The corresponding correlations with the force-Gram view are considerably
lower: \(0.2178\), \(0.1715\), and \(0.2743\).

The selected Multi-GWMDS projection is even more strongly associated
with the interatomic-distance view. Its correlations with this view
reach \(0.7506\), \(0.7928\), and \(0.8204\), whereas its correlations
with the force-Gram view decrease to \(0.0624\), \(0.0524\), and
\(0.0669\). Consequently, the stronger preservation of the dominant
interatomic-distance view by Inductive Multi-GWMDS is accompanied by a
substantial loss of information from the force-based representation.

Mean-GWMDS provides a more balanced multi-view compromise. It accepts a
small reduction in the preservation of interatomic distances in exchange
for markedly stronger agreement with the force-Gram relations. This
behavior explains why Inductive Mean-GWMDS achieves the best
view-averaged results under every relational geometry, despite not
attaining the highest correlation with the dominant view individually.

\paragraph{Projection selection.}
Table~\ref{tab:rmd17_projection_selection} reports the training-only
Pearson correlations of the two view-dependent Multi-GWMDS barycentric
projections with each molecular view. As expected, each projection
preserves primarily the relational geometry of the view from which it
is derived. Nevertheless, when the correlations are averaged uniformly
across the two views, the interatomic-distance projection achieves the
highest selection score under all three relational geometries: \(0.4642\),
\(0.4762\), and \(0.5153\) for Euclidean, geodesic, and cosine relations,
respectively. The corresponding scores of the Force-Gram projection are
\(0.3414\), \(0.3300\), and \(0.4199\), yielding selection margins of
\(0.1228\), \(0.1462\), and \(0.0954\).

\begin{table*}[t]
    \centering
    \caption{Training-only evaluation of the Multi-GWMDS barycentric projections on rMD17-Aspirin. For each candidate projection, $\rho_1$ and $\rho_2$ denote the Pearson correlations with the interatomic-distance and Force-Gram views, respectively. The mean is computed uniformly across the two views and used as the selection score. The selected projection is shown in bold. The last row reports the selected-minus-other difference.}
    \label{tab:rmd17_projection_selection}

    \resizebox{\textwidth}{!}{%
    \begin{tabular}{@{}lccccccccc@{}}
        \toprule
        & \multicolumn{3}{c}{Euclidean}
        & \multicolumn{3}{c}{Geodesic}
        & \multicolumn{3}{c}{Cosine} \\
        \cmidrule(lr){2-4}
        \cmidrule(lr){5-7}
        \cmidrule(lr){8-10}

        Barycentric projection
        & $\rho_1$ & $\rho_2$ & Mean
        & $\rho_1$ & $\rho_2$ & Mean
        & $\rho_1$ & $\rho_2$ & Mean \\
        \midrule

        Interatomic-distance view
        & 0.7966 & 0.1318 & \textbf{0.4642}
        & 0.8468 & 0.1057 & \textbf{0.4762}
        & 0.8610 & 0.1695 & \textbf{0.5153} \\

        Force-Gram view
        & 0.0926 & 0.5902 & 0.3414
        & 0.0682 & 0.5918 & 0.3300
        & 0.1195 & 0.7203 & 0.4199 \\
        \midrule

        Selected-minus-other difference
        & 0.7040 & $-0.4584$ & 0.1228
        & 0.7786 & $-0.4861$ & 0.1462
        & 0.7415 & $-0.5508$ & 0.0954 \\
        \bottomrule
    \end{tabular}%
    }
\end{table*}
Unlike the nearly tied Euclidean selection observed for ERA5, the
rMD17 selection is unambiguous under the present split. Importantly,
this result does not indicate that the selected projection is superior
for both views individually. Rather, its larger advantage with respect
to the interatomic-distance geometry outweighs its lower agreement with
the Force-Gram geometry under the mean-based selection criterion.

\paragraph{GW objective and sample-indexed preservation.}
Table~\ref{tab:rmd17_objective_mismatch} compares the final
Multi-GWMDS teacher objective with that of direct neural multi-view GW.

\begin{table}[h]
    \centering
    \caption{Final GW objectives and view-averaged test Pearson
    correlations on rMD17-Aspirin.}
    \label{tab:rmd17_objective_mismatch}
    \small
    \begin{tabular}{@{}lccc@{}}
        \toprule
        & Euclidean
        & \shortstack{Geodesic}
        & Cosine \\
        \midrule

        \shortstack[l]{Multi-GWMDS teacher objective}
        & 0.035079
        & 0.025891
        & 0.009187 \\

        \shortstack[l]{Direct GW objective}
        & 0.036273
        & 0.025977
        & 0.009195 \\

        \shortstack[l]{Inductive Multi-GWMDS test Pearson }
        & 0.4065
        & 0.4226
        & 0.4437 \\

        \shortstack[l]{Direct GW test Pearson }
        & 0.3540
        & 0.3471
        & 0.2942 \\

        \bottomrule
    \end{tabular}
\end{table}

Under geodesic relations, the Multi-GWMDS teacher and direct-GW
objectives are \(0.025891\) and \(0.025977\), respectively. Under cosine
relations, they are \(0.009187\) and \(0.009195\). Their relative
differences are approximately \(0.33\%\) and \(0.09\%\), indicating
that both methods attain almost identical structural objectives. The
Euclidean relative difference is somewhat larger, approximately
\(3.4\%\), but the final values remain of the same order.

Despite this objective-level agreement, the out-of-sample
representations differ substantially. In the cosine experiment,
Inductive Multi-GWMDS achieves a test Pearson correlation of \(0.4437\),
compared with only \(0.2942\) for direct multi-view GW, producing a
difference of \(0.1495\). The weaker direct-GW result therefore cannot
be explained merely by insufficient minimization of its training
objective.

This result reinforces the distinction between structural agreement
under an optimized coupling and sample-indexed prediction. A small GW
objective can be obtained through a coupling that reorganizes the
observations, whereas a neural predictor must associate each input
conformation with the appropriate latent point. Barycentric distillation
explicitly supplies this sample-indexed supervision, while the direct
objective does not uniquely determine the required correspondence.

\paragraph{Computational cost.}
Table~\ref{tab:rmd17_computational_cost} reports the execution times
averaged across the three relational geometries.

\begin{table*}[h]
    \centering
    \caption{Average execution times on rMD17-Aspirin. Teacher and
    student times are reported separately to distinguish the initial
    relational optimization from neural distillation.}
    \label{tab:rmd17_computational_cost}
    \small
    \begin{tabular}{llr}
        \toprule
        Formulation & Optimization stage & Time (s) \\
        \midrule

        Inductive Mean-GWMDS
        & Mean-GWMDS teacher
        & 323.59 \\

        Inductive Mean-GWMDS
        & Consensus-target student
        & 1.63 \\

        Inductive Multi-GWMDS
        & Multi-GWMDS teacher
        & 607.56 \\

        Inductive Multi-GWMDS
        & Selected-projection student
        & 1.36 \\

        Direct multi-view GW
        & Direct neural optimization
        & 546.63 \\
        \bottomrule
    \end{tabular}
\end{table*}

The Mean-GWMDS teacher requires an average of \(323.59\) seconds,
whereas the Multi-GWMDS teacher requires \(607.56\) seconds because a
separate transport plan is updated for each view. The corresponding
student-training times are only \(1.63\) and \(1.36\) seconds.
Distillation therefore represents less than \(0.6\%\) of the
corresponding teacher-optimization time.

The complete Inductive Mean-GWMDS pipeline requires approximately
\(325.22\) seconds on average, while direct neural multi-view GW requires
\(546.63\) seconds. Thus, the direct formulation requires approximately
\(1.68\) times the runtime of the Inductive Mean-GWMDS pipeline. The
complete Inductive Multi-GWMDS pipeline requires approximately
\(608.92\) seconds and is therefore about \(11.4\%\) slower than direct
multi-view GW during initial training. 

\subsection{Limitations}
\label{app:limitations}

The proposed framework transfers out-of-sample prediction to a neural
student, but its teacher remains transductive and requires dense
pairwise relational matrices and transport plans. Consequently, the
training-stage memory requirement grows quadratically with the number
of observations, and repeated GW updates remain expensive for large
datasets.


\end{document}